\documentclass[letterpaper,10pt,twocolumn]{article}
\usepackage[T1]{fontenc}
\usepackage[textwidth=7in,textheight=9in,top=0.75in,columnsep=0.375in]{geometry}
\usepackage{newtxtext}
\usepackage{helvet,courier}
\usepackage[hyphens]{url}
\usepackage{graphicx}
\usepackage[authoryear,round]{natbib}
\usepackage{caption}
\usepackage{booktabs,array}
\usepackage{amsmath,amssymb,amsthm}
\usepackage{microtype}
\usepackage[unicode,hidelinks]{hyperref}
\newtheorem{theorem}{Theorem}
\newtheorem{proposition}[theorem]{Proposition}
\newtheorem{corollary}[theorem]{Corollary}
\newtheorem{lemma}[theorem]{Lemma}
\theoremstyle{definition}
\newtheorem{assumption}{Assumption}

\theoremstyle{remark}
\newtheorem{remark}{Remark}
\theoremstyle{plain}
\newcommand{\E}{\mathbb{E}}
\newcommand{\Pp}{\mathbb{P}}
\newcommand{\R}{\mathbb{R}}
\newcommand{\Aset}{\mathcal{A}}
\newcommand{\Asup}{\mathcal{A}_{\eta}}
\newcommand{\ind}{\perp\!\!\!\perp}
\newcommand{\TV}{\operatorname{TV}}
\newcommand{\doop}{\operatorname{do}}
\newcommand{\1}{\mathbf{1}}
\newcommand{\norm}[1]{\left\lVert #1\right\rVert}
\newcommand{\CuratedN}{99}
\newcommand{\ActionSupportedN}{80}
\newcommand{\SupportedN}{79}

\title{CaMeWM: A Causal Medical World Model for Generative Tumor Evolution and Treatment Planning}
\author{%
\begin{minipage}{\dimexpr\textwidth-2\tabcolsep\relax}
\centering
{\large Huanlong Gao, Xuelei He\textsuperscript{*}, Fengjun Zhao and Xiaowei He\textsuperscript{*}}\\[0.6em]
{\small School of Electronic Information, Northwest University, No.~1 Xuefu Avenue, Xi'an, 710127, China\endgraf}
\vspace{0.25em}
{\small Xi'an Key Lab of Radiomics and Intelligent Perception, Northwest University, No.~1 Xuefu Avenue, Xi'an, 710127, China\endgraf}
\vspace{0.5em}
{\small\textsuperscript{*}Co-corresponding authors.\endgraf}
\end{minipage}%
}
\date{}
\hypersetup{pdftitle={CaMeWM: A Causal Medical World Model for Generative Tumor Evolution and Treatment Planning},pdfauthor={Huanlong Gao; Xuelei He; Fengjun Zhao; Xiaowei He}}
\begin{document}
\let\cite\citep
\maketitle
\thispagestyle{plain}

\begin{abstract}
Treatment planning is intrinsically counterfactual: clinicians must compare how a patient's tumor and survival would evolve under competing regimens, whereas a retrospective cohort records only the regimen received and one follow-up scan. A treatment-conditioned generator trained on such data can therefore reproduce historical treatment allocation rather than simulate the interventional disease dynamics needed to rank regimens. We introduce \textbf{CaMeWM}, a causal medical world model for supported post-TACE tumor evolution and treatment planning. CaMeWM implements one intervention-to-decision loop rather than three independent prediction tasks: it first admits a queried regimen only for clinically comparable histories with empirical treatment support; it then generates a distribution of lesion-local follow-up worlds through an action-conditioned response descriptor, deterministic anchor, and residual 3D diffusion transition; and it finally maps sampled worlds to restricted mean survival time (RMST) with a response-Cox utility model. Under pre-treatment adjustment and positivity, this loop estimates a prespecified, support-restricted g-formula target; all selection and nuisance estimation are fold-local, and the induced policy is evaluated by cross-fitted, censoring-adjusted doubly robust estimation. Across 79 paired out-of-fold episodes, CaMeWM achieves 0.618 SSIM and 0.528 change-map SSIM for follow-up CT transition, while its generated-world survival pathway attains a 0.615 C-index; factual-only AIPW estimates a 9.60-week RMST gain (95\% CI 0.61--18.13) for the frozen policy. CaMeWM provides an inspectable basis for retrospective, support-restricted treatment planning.
\end{abstract}

\section{Introduction}
Treatment planning is intrinsically counterfactual.  When choosing among transarterial chemoembolization (TACE) regimens, clinicians need to compare how an otherwise similar patient's tumor and survival would evolve under each regimen.  Longitudinal CT makes the question inspectable, but each retrospective episode records only the regimen delivered and one follow-up state.  The alternative treatment worlds needed for comparison are unobserved.

Generative medical world models offer an appealing route to this problem.  Rather than predict prognosis alone, they can render a treatment-conditioned future tumor state and connect it to survival-based decision making.  Medical World Model (MeWM) established this direction by coupling an action policy, a tumor generator, and inverse survival dynamics \cite{yang2025mewm}.  However, visual plausibility or action conditioning alone is insufficient for retrospective treatment planning: a model can generate a credible follow-up CT while answering the wrong clinical question.

The obstacle is causal.  Let $H$ denote pre-treatment history, $A$ a complete regimen, and $X^+$ the follow-up CT.  In observational TACE cohorts, baseline tumor burden, hepatic function, and clinical status influence treatment selection and outcome.  Hence $p(X^+\mid H,A=a)$ need not equal $p(X^+\mid H,\doop(A=a))$ \cite{hernan2020whatif,shalit2017ite}.  Optimizing the former can reproduce historical allocation rather than treatment effect, or extrapolate to a regimen rarely observed for comparable patients.  Censoring creates a second problem: a generated state must also be mapped to a survival utility with appropriate adjustment.

A causal medical world model is therefore more than a conditional generator followed by a risk score.  As a computational system, it must couple three operations: determine whether a regimen can be simulated for a patient's pre-treatment history; synthesize the regimen-specific local transition while retaining both response-associated and direct treatment effects; and score sampled worlds without allowing held-out patients to influence the learned propensity, transition, or outcome functions.  Omitting any operation yields either unsupported extrapolation, an observational image predictor, or an optimistically evaluated decision rule.

We introduce \textbf{CaMeWM}, a causal medical world model for support-restricted post-TACE tumor evolution and treatment planning.  Its first stage admits an intervention query only when a complete regimen has fold-local joint-propensity support for the observed pre-treatment history.  Its second stage constructs an action-specific distribution of follow-up worlds: an auxiliary response descriptor captures coarse treatment response, a deterministic anchor preserves stable anatomy, and lesion-local residual diffusion represents remaining stochastic transition while retaining a direct treatment route.  Its third stage maps factual or sampled worlds to RMST with a response-Cox utility model and evaluates the resulting policy out of fold.  Thus, the causal content comes from the target intervention, pre-treatment adjustment, and support rule; the generator makes that target inspectable rather than creating causal validity by architecture alone.

CaMeWM is fitted fold-locally and evaluated as a decision system rather than as an image synthesizer alone.  Cross-fitted, censoring-adjusted doubly robust (DR) estimation evaluates the supported policy on held-out patients.  We distinguish factual-state diagnostics from the planning-time generated-world pathway, and use semi-synthetic stress tests for assumptions that retrospective data cannot verify.

Our contributions are threefold:
\begin{itemize}
  \item We formalize a support-restricted interventional world target for complete TACE regimens, operationalizing $\doop(A=a)$ only where fold-local joint propensity satisfies empirical positivity.
  \item We instantiate its g-formula with a response-grounded anchor--residual diffusion transition and a response-Cox RMST utility.  The simulator retains direct treatment effects, renders a local follow-up CT distribution, and exposes an interpretable response descriptor without claiming causal mediation.
  \item We provide a leakage-free decision protocol that ranks supported sampled worlds by RMST and uses censoring-adjusted, cross-fitted DR evaluation, matched causal baselines, transport diagnostics, controlled ablations, and semi-synthetic stress tests.
\end{itemize}
CaMeWM targets an intervention-specific distribution and expected utility within observed support.  It does not identify a unique alternate-world CT for an individual patient, nor does retrospective evaluation establish prospective clinical benefit.

\section{Related Work}
\paragraph{Generative World Models.}
World models simulate future states to support decision-making before action execution \cite{ha2018worldmodels,hafner2023mastering}. Medical extensions include radiographic representation learning \cite{yue2025chexworld}, motion-aware ultrasound guidance \cite{yue2025echoworld}, action-controllable surgical video synthesis \cite{koju2025surgical}, and longitudinal EHR event forecasting \cite{kraljevic2024foresight}. For treatment-conditioned imaging, MeWM uses diffusion to synthesize post-treatment tumor evolution \cite{yang2025mewm}, while CLARITY models context-aware disease trajectories in latent space for survival-aware evaluation \cite{ding2025clarity}. Causal-generative methods further model intervention-indexed outcomes from longitudinal histories \cite{alinezhad2026cdm,frauen2023deepace}. CaMeWM complements these lines by coupling follow-up CT simulation with a support-restricted causal target and factual-only cross-fitted AIPW evaluation.

\paragraph{Causal treatment learning.}
Balanced representations estimate heterogeneous effects under ignorability \cite{shalit2017ite}, while DeepACE combines longitudinal g-computation with targeted estimation \cite{frauen2023deepace}.  DR estimators remain consistent when either the propensity or outcome regression is correct \cite{bang2005dr}, and cross-fitting limits overfitting bias from learned nuisance functions \cite{chernozhukov2018dml}.  These methods motivate our adjustment and evaluation strategy, but they do not expose a generated 3D post-treatment state.

\paragraph{Survival-aware planning.}
Cox, DeepSurv, and discrete-time networks are common models for censored outcomes \cite{cox1972regression,lee2018deephit}.  RMST provides an interpretable utility without requiring proportional hazards \cite{royston2013rmst}.  CaMeWM integrates this utility with support-constrained tumor-world sampling, rather than estimating treatment effects from tabular trajectories alone.

\section{CaMeWM: Supported Interventional Worlds}
\paragraph{Overview.}
Figure~\ref{fig:method} organizes CaMeWM as a three-stage intervention-to-decision system.  Stage~I admits a regimen query only on empirical support; Stage~II constructs an action-specific distribution of follow-up worlds; and Stage~III assigns RMST utility and ranks supported actions.  This is an assumption-gated population identification claim: the SCM and pre-treatment adjustment identify the target under consistency, conditional exchangeability, and positivity, while the learned networks estimate observed-law components and cannot verify those assumptions.  Cross-fitting, censoring adjustment, and DR estimation form an external evaluation rail for the induced policy.

\begin{figure*}[t]
  \centering
  \includegraphics[width=0.94\textwidth]{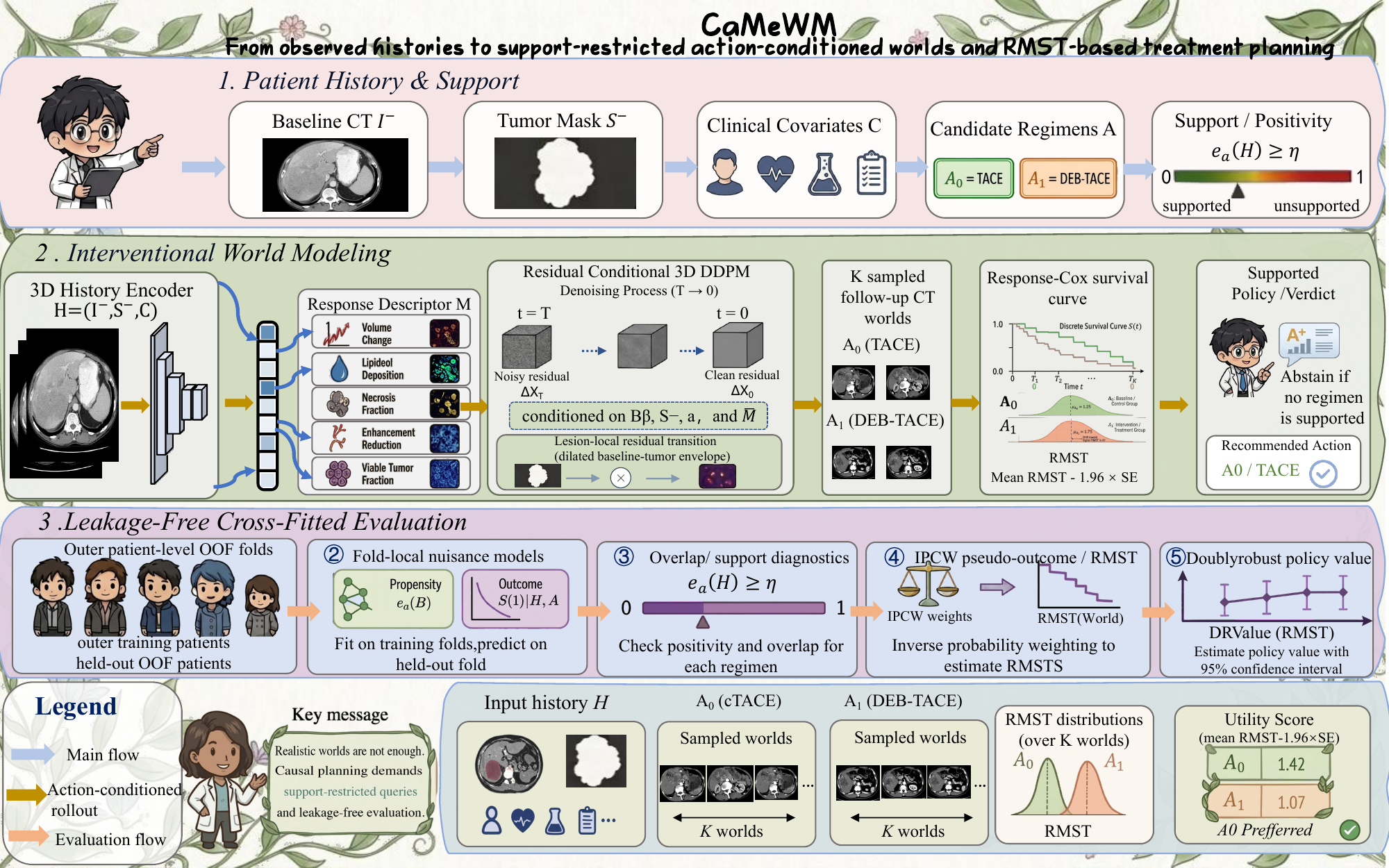}
  \caption{CaMeWM replaces observational assignment by a supported $\doop(A=a)$ query. Under consistency, conditional exchangeability, and positivity, the g-formula identifies the regimen-specific population world law and its RMST value on empirical support (Theorem~\ref{thm:main-identification}; full assumptions and proof in the supplement). The response code, deterministic anchor, residual DDPM, and survival head estimate components of this identified functional; their architectures neither establish exchangeability nor identify a unique patient's counterfactual CT. The lower rail performs leakage-free IPCW/DR policy evaluation. CT thumbnails are schematic, not patient data or results.}
  \label{fig:method}
\end{figure*}

\subsection{Stage I: Supported Interventional Target}
For patient $i$, let
\begin{equation}
 O_i=(H_i,A_i,M_i,X_i^+,\widetilde T_i,\Delta_i),
 \quad H=(I^-,S^-,C),
 \label{eq:observed}
\end{equation}
where $X^-=(I^-,S^-)$ is the baseline CT--mask state, $C$ contains measured pre-treatment clinical covariates, $A$ is a complete regimen, and $X^+=(I^+,S^+)$ is the registered post-treatment state. Event time is observed as $\widetilde T=\min(T,T^{\rm cen})$ with indicator $\Delta$, where $T^{\rm cen}$ is censoring time. For tumor counts $n^-,n^+$ and masked HU values $z^-,z^+$, $M=s(X^-,X^+)\in\mathbb R^5$ comprises relative tumor-count change, the high- and low-attenuation fractions, mean attenuation change, and viable-tumor fraction. Thus $M$ is a deterministic same-visit summary used as an auxiliary generation code; it is neither a pre-treatment confounder nor a biological mediator, and no front-door or path-specific claim is made through it.

The causal graph over substantive variables allows measured history to affect treatment, post-treatment state, and survival, while retaining direct treatment effects:
\begin{equation}
\begin{aligned}
 H&\rightarrow A, & H&\rightarrow X^+, & H&\rightarrow T,\\
 A&\rightarrow X^+, & A&\rightarrow T, & X^+&\rightarrow T.
\end{aligned}
 \label{eq:dag}
\end{equation}
For support threshold $\eta$, define $\mathcal A_\eta(h)=\{a:e_a(h)\ge\eta\}$.  For $a\in\mathcal A_\eta(h)$, $\doop(A=a)$ removes the observational factor $p(A\mid H)$ and replaces it by a point mass at $a$, while leaving the post-treatment mechanisms invariant.  Our target is the resulting world law.  Assume consistency and no interference, conditional exchangeability $A\perp\{M(a),X^+(a),U_\tau(a)\}\mid H$, and positivity on $\mathcal A_\eta(h)$.  Then
\begin{align}
 P_h^a(dm,dx)&=p(m,x\mid h,\doop(A=a))\notag\\
 &\overset{\rm id}{=}p(m,x\mid h,A=a)\notag\\
 &=p_M(m\mid h,a)\,p_X(x\mid h,a,m).
 \label{eq:factorization}
\end{align}
The marked equality is the causal identification step; the last equality is only an estimation ordering of the augmented joint law, not a causal arrow $M\to X^+$. The implemented response network replaces $p_M$ by a point mass, yielding a plug-in approximation that omits response-code uncertainty. If $H$ omits a treatment--outcome confounder, identification fails regardless of diffusion accuracy.

\paragraph{Supported action interface.}
We fit the joint propensity $e_a(h)=P(A=a\mid H=h)$ using only pre-treatment variables and treat each complete regimen as one action.  The fold-local classifier encodes the baseline CT and tumor mask with three 3D stages (16, 32, and 64 channels), concatenates this representation with eight clinical covariates and their missingness indicators, and predicts the two regimen probabilities through a 128-dimensional softmax head.  It is trained by joint cross-entropy and selected by inner-validation cross-entropy; no post-treatment state or outcome enters this model, and we apply no post-hoc probability calibration.  The resulting probability determines which worlds may be queried; outcome training also uses $w_i=\min\{10,1/\widehat e_{A_i}(H_i)\}$ \cite{austin2015iptw}.
We plan in clinically interpretable RMST units and normalize only for bounded theory:
\begin{equation}
\begin{aligned}
 Y_\tau(a)&=\min\{T(a),\tau\}\in[0,\tau],\\
 V_\tau(a,h)&=\E[Y_\tau(a)\mid H=h],\\
 U_\tau(a)&=Y_\tau(a)/\tau,\quad
 \overline V_\tau(a,h)=V_\tau(a,h)/\tau.
\end{aligned}
\label{eq:utility}
\end{equation}
Here $\tau=156.4286$ weeks (36 months). Planning scores, pseudo-outcomes, and reported gains use $V_\tau$ in weeks; only the stability bound uses $\overline V_\tau\in[0,1]$ \cite{royston2013rmst}.

\subsection{Stage II: Response-Grounded World Simulation}
\paragraph{Fold-local state completion.}
Segmentation measures state; it does not adjust confounding.  Because manual follow-up masks are incomplete, outer fold $k$ trains $g_\psi$ only on fold-training baseline CT/mask pairs and selects it by validation Dice.  It constructs factual $S^+=g_\psi(I^+)$ for response and survival training and generated $\widehat S^+=g_\psi(\widehat I^+)$ during rollout.  Thus
\begin{equation}
 p_X(i^+,s^+\mid h,a,m)
 =p_\theta(i^+\mid h,a,m)\,\delta_{g_\psi(i^+)}(s^+),
 \label{eq:state_transition}
\end{equation}
where $\delta_b$ is a point mass at $b$, not a joint CT--mask diffusion model.  Masks below 16 voxels trigger a fixed intensity fallback recorded in QC.  This fold-local construction prevents post-treatment test images or outcomes from fitting the segmentation nuisance, although the resulting follow-up masks remain proxy labels rather than manual ground truth.

\paragraph{Auxiliary response descriptor.}
The response vector is an inspectable auxiliary code in the coarse-to-fine parameterization of $p(m,x\mid h,a)$, not a causal mediator.  The network $f_\phi$ encodes $(I^-,S^-)$ and concatenates image features with $C$ and a one-hot action.  Fold-standardized targets use smooth $L_1$ loss,
\begin{equation}
 \widehat M(a)=f_\phi(H,a),\qquad
 \mathcal L_M=\operatorname{SmoothL1}(\widehat M,M).
 \label{eq:response}
\end{equation}
After inverse standardization and range clamping, this implements $\widehat P_M(\cdot\mid h,a)=\delta_{\widehat M(h,a)}$ and therefore omits code uncertainty. Predicted rather than factual codes are used with probability $0.5$ during diffusion training; validation and intervention rollouts always use predicted codes, so the factual response code never enters a reported policy.

\paragraph{Deterministic anchor and pseudo-mask reliability.}
Before stochastic refinement, an action- and response-conditioned anchor
$B_\beta(H,a,\widehat M(a))$ predicts a deterministic coarse follow-up CT.  It
captures the stable, large-scale treatment-conditioned state, leaving the
diffusion model to represent local residual variation.  ``Deterministic anchor
only'' in the ablation uses $B_\beta$ directly and removes the residual DDPM;
it is therefore a test of stochastic residual refinement rather than a removal
of the causal estimand, propensity restriction, or treatment conditioning.

Follow-up tumor masks are fold-local pseudo labels.  For patient $i$, we form a
reliability score $c_i\in[0,1]$ as the mean product of test-time-augmentation
vote agreement and probability margin in the predicted-or-baseline lesion
region.  Let $\widetilde S_i^+$ denote the resulting pseudo follow-up mask.  We
use its confidence-weighted union with the baseline mask,
\begin{equation}
 W_i=S_i^-+c_i\left[(S_i^-\vee\widetilde S_i^+)-S_i^-\right],
 \label{eq:confidence-mask}
\end{equation}
for lesion-weighted reconstruction, residual, and semantic losses, where $\vee$ and $\odot$ denote voxel-wise union and multiplication.  Thus,
low-confidence pseudo follow-up regions contribute less beyond $S_i^-$.  The
``uniform pseudo-mask weighting'' ablation fixes $c_i=1$ for every patient; it
does not remove pseudo masks, but instead treats all of them as equally
reliable.

Registered baseline and follow-up scans share most anatomy, so the model predicts local change instead of reconstructing the complete volume.  We define the anchor residual $R=(I^+-B_\beta)/2$ and its noisy version $R_r=\sqrt{\bar\alpha_r}R+\sqrt{1-\bar\alpha_r}\epsilon$, where $\epsilon\sim\mathcal N(0,I)$ and $r\in\{1,\ldots,200\}$ indexes a linear noise schedule.  A conditional 3D denoiser $\epsilon_\theta(R_r,B_\beta,S^-,r,A,M)$ predicts the noise.  Decoding forms $\widehat I^+=B_\beta+2\widehat R$ and gates residual change with a $9\times9\times9$ max-pooled baseline-tumor envelope (a four-voxel radius at the resampled 1-mm spacing).  Beyond noise-prediction MSE, the objective emphasizes global fidelity, tumor fidelity, and structure:
\begin{align}
 \mathcal L_X={}&\|\epsilon-\epsilon_\theta\|_2^2
 +\lambda_g\|\widehat I^+-I^+\|_1 \\
 &+\lambda_t\frac{\|W\odot(\widehat I^+-I^+)\|_1}
 {\|W\|_1+\varepsilon} \\
 &+\lambda_s(1-\operatorname{SSIM}_3),
 \label{eq:diffusion}
\end{align}
where $(\lambda_g,\lambda_t,\lambda_s)=(0.5,2,0.25)$, $\varepsilon=10^{-6}$ prevents division by zero, and $\operatorname{SSIM}_3$ is three-dimensional structural similarity \citep{wang2004ssim}.  These predictive losses estimate the two observational factors on the right side of Eq.~\eqref{eq:factorization}; they acquire an intervention interpretation only through its adjustment assumptions and support restriction.  The hard spatial gate limits unsupported global edits but also assumes that treatment-induced change lies inside the dilated baseline-tumor envelope.

\subsection{Stage III: Utility-Aware Supported Rollout}

\paragraph{Survival utility.}
For each factual or generated post-treatment state, we form a fold-standardized response vector $Z$ comprising the regimen, predicted response code, pre-treatment clinical covariates and missingness indicators, and eight lesion-local state statistics.  A response-Cox model maps $Z$ to a log-risk score; its $\ell_2$ penalty is selected on the inner validation partition by C-index.  With Breslow baseline hazard $\widehat H_0$ fitted on outer-fold training outcomes, it produces $\widehat S(t\mid Z)=\exp\{-\widehat H_0(t)\exp r_\omega(Z)\}$ and $\widehat R_\tau(Z)=\int_0^\tau\widehat S(t\mid Z)dt$.  The utility model is fitted on factual states $(I^+,g_\psi(I^+))$, whereas rollout substitutes predicted responses and sampled states $(\widehat I^+,g_\psi(\widehat I^+))$ and averages RMST over $K=8$ worlds.  This avoids training the outcome head on self-generated labels, but it leaves a factual-to-generated covariate shift; factual OOF calibration therefore cannot certify generated-state calibration.

\paragraph{Fold-local staged learning.}
The six modules form one estimator but not one gradient update. Within outer fold $k$, validation-selected segmentation is fitted first; propensity and response models are then fitted independently; the frozen response model supplies predicted codes to the anchor; the anchor initializes residual diffusion; and frozen segmentation and propensity models supply masks and weights to survival.  The parameter blocks $(\psi,\gamma,\phi,\beta,\theta,\omega)$ denote the segmenter, propensity model, response-descriptor predictor, anchor, residual denoiser, and response-Cox utility model, respectively. Symbolically,
\begin{equation}
 \widehat\psi_k\rightarrow(\widehat\gamma_k,\widehat\phi_k)
 \rightarrow\widehat\beta_k\rightarrow\widehat\theta_k
 \rightarrow\widehat\omega_k
 \rightarrow\text{OOF rollout/AIPW}.
 \label{eq:staged}
\end{equation}
Survival loss does not update the generator, and propensity, censoring, and factual-only AIPW evaluation remain external to policy fitting. This separation enforces fold-external nuisance prediction; it does not itself provide causal identification.

\paragraph{Supported intervention rollout.}
Observed data follow $A\sim e(\cdot\mid h)$; CaMeWM instead fixes $a$ when $\widehat e_a(h)\ge\eta=0.05$.  This is the implemented $\doop(A=a)$ operation.  It predicts $\widehat M(a)$, samples $K=8$ CTs, obtains RMST values $R_1(a),\ldots,R_K(a)$, and selects
\begin{equation}
 \widehat\pi(h)=\arg\max_{a:\widehat e_a(h)\ge\eta}
 \left\{\overline R(a)-1.96\,s_R(a)/\sqrt K\right\}.
 \label{eq:planning}
\end{equation}
Here $\overline R(a)=K^{-1}\sum_{j=1}^K R_j(a)$ and $s_R(a)$ is the sample standard deviation across worlds.  This is a conservative Monte Carlo ranking score.  We do not interpret it as a simultaneously calibrated finite-sample confidence bound over patients and actions.

\subsection{Cross-Fitted Evaluation and Theoretical Guarantee}
Let $Q_\tau(h,a,m,x)=\E[Y_\tau\mid h,a,m,x]$ be conditional RMST in weeks and $\overline Q_\tau=Q_\tau/\tau\in[0,1]$.
For probability distributions $P$ and $Q$, write $\TV(P,Q)=\sup_B|P(B)-Q(B)|$, where the supremum ranges over measurable events $B$.

\begin{theorem}[Supported interventional-world identification and oracle stability]
\label{thm:main-identification}
Under consistency, conditional exchangeability, and positivity, for every supported $(a,h)$,
\begin{align}
 V_\tau(a,h)=\iint &Q_\tau(h,a,m,x)p_X(x\mid h,a,m)\notag\\
 &p_M(m\mid h,a)\,dx\,dm.
 \label{eq:gformula}
\end{align}
If $\sup\TV(P_M,\widehat P_M)\le\epsilon_M$, $\sup\TV(P_X,\widehat P_X)\le\epsilon_X$, and $\sup|\overline Q_\tau-\widehat{\overline Q}_\tau|\le\epsilon_Q$, then
\begin{equation}
 \sup_{a,h}|\overline V_\tau(a,h)-\widehat{\overline V}_\tau(a,h)|
 \le\epsilon_M+\epsilon_X+\epsilon_Q,
 \label{eq:errorbound}
\end{equation}
and normalized regret on the same support is at most twice this bound; multiplying by $\tau$ converts both statements to weeks.
\end{theorem}
\noindent\emph{Scope of the bound.}
Equation~\eqref{eq:errorbound} is an oracle population stability statement conditional on the stated factor-wise errors; it is not a finite-sample certificate for the fitted implementation.  In particular, the retrospective cohort does not identify $\epsilon_M$, $\epsilon_X$, or $\epsilon_Q$.

\noindent\emph{Proof sketch.}
Exchangeability identifies the intervention by adjustment over $H$; consistency and the chain rule yield Eq.~\eqref{eq:gformula}. Adding and subtracting one learned factor at a time gives Eq.~\eqref{eq:errorbound}; inserting the estimated maximizer between oracle values gives the regret bound. We instead report factor-wise OOF diagnostics and test exact behavior only under semi-synthetic ground truth. Full assumptions and proofs are in the supplement.

Evaluation additionally assumes independent censoring through $\tau$ under a fold-training marginal censoring model \citep{hernan2020whatif}. Each outer fold supplies nuisance predictions for held-out patients. If $k(i)$ is patient $i$'s fold, $\widehat G_{-k(i)}(t-)$ is the fold-training Kaplan--Meier estimate of marginal censoring survival, giving
\begin{equation}
 Y_i^\tau=\int_0^\tau
 \frac{\1(\widetilde T_i\ge t)}
 {\max\{\widehat G_{-k(i)}(t-),10^{-3}\}}\,dt.
 \label{eq:ipcw_rmst}
\end{equation}
For policy evaluation, $\pi_i=\widehat\pi(H_i)$ is the already frozen CaMeWM recommendation.  Separately, $\widehat Q^{\mathrm{fact}}_{\tau,a}(H_i)$ denotes a fold-local clinical RMST nuisance fit only from pre-treatment clinical history, factual regimen, and observed survival in the outer training fold.  It excludes $X^+$, response descriptors, generated worlds, pseudo masks, and all post-treatment information.  With this factual nuisance,
\begin{align}
 \omega_i^\pi&=\frac{\1(A_i=\pi_i)}
 {\max\{\widehat e_{A_i}(H_i),0.01\}},\notag\\
 \widehat V_{\rm AIPW,\tau}(\pi)&=\frac1n\sum_i
 \widehat\psi_i^{\mathrm{fact}}(\pi),\notag\\
 \widehat\psi_i^{\mathrm{fact}}(\pi)&=
 \widehat Q^{\mathrm{fact}}_{\tau,\pi_i}(H_i)
 +\omega_i^\pi\left\{Y_i^\tau-
 \widehat Q^{\mathrm{fact}}_{\tau,A_i}(H_i)\right\}.
 \label{eq:aipw}
\end{align}
Subject to the causal assumptions and a correct censoring model, this factual-only AIPW estimator is consistent if either its treatment propensity or its factual outcome regression is correct \cite{bang2005dr,chernozhukov2018dml}.  The world model constructs $\widehat\pi$ but is not an outcome nuisance in Eq.~\eqref{eq:aipw}; consequently, the classical double-robustness statement applies to the stated treatment/outcome nuisance pair.  It neither validates generated-state calibration nor repairs unmeasured confounding.

\section{Experiments}
\subsection{Data and Leakage-Free Protocol}
\paragraph{Cohort and imaging state.}
We use the public HCC-TACE-Seg collection of paired pre/post-TACE CT, treatment, response, and survival records \cite{moawad2023hcctace}.  We predeclare two represented joint regimens and retain \SupportedN{} outcome-blind eligible episodes (50 TACE and 29 DEB-TACE) after excluding one episode with an empty released tumor mask; eligibility and imaging QC do not use survival outcomes.  Baseline label 2 (Mass/HCC tumor) defines the lesion state.  Follow-up CT is rigidly registered to the baseline grid, resampled to 1-mm spacing, and evaluated in a shared $96^3$ physical crop; fold-local segmenters supply follow-up masks when manual masks are unavailable.  The complete label audit, registration contract, and QC artifacts are in Supplementary.

\paragraph{Protocol.}
We use patient-level five-fold outer CV, stratified by regimen and event/time strata, with an internal validation split within each outer-training partition.  All preprocessing statistics, segmenter, propensity, response, transition, and survival models are fit fold-locally; each patient receives one OOF prediction.  Checkpoints are selected only on validation data.  We report 2,000 patient-bootstrap intervals and paired differences using identical resampled patients.  Full selection rules, training settings, and hardware details are deferred to Supplementary.
\begin{table*}[t]
\centering
{\small
\setlength{\tabcolsep}{5.5pt}
\begin{tabular}{@{}lcccc@{}}
\toprule
Method & Tumor HU-MAE $\downarrow$ & $\Delta$SSIM $\uparrow$ & PSNR $\uparrow$ & SSIM $\uparrow$ \\
\midrule
Persistence & 57.192 [52.069, 63.049] & 0.157 [0.135, 0.180] & 16.877 [16.148, 17.661] & 0.469 [0.443, 0.497] \\
3D U-Net & 26.944 [23.224, 31.104] & 0.523 [0.481, 0.566] & 19.157 [18.359, 19.974] & 0.610 [0.584, 0.637] \\
CVAE-3D & 32.300 [28.065, 37.105] & 0.508 [0.464, 0.549] & 18.815 [18.027, 19.670] & 0.587 [0.562, 0.612] \\
Mask-DDPM & 63.721 [58.270, 69.962] & 0.051 [0.026, 0.080] & 16.665 [15.966, 17.362] & 0.451 [0.425, 0.477] \\
MeWM$^\ast$  & 62.910 [57.562, 68.555] & 0.074 [0.048, 0.102] & 16.684 [15.993, 17.373] & 0.438 [0.412, 0.465] \\
\textbf{CaMeWM} & \textbf{27.145 [23.139, 31.534]} & \textbf{0.528 [0.484, 0.569]} & \textbf{19.193 [18.353, 20.112]} & \textbf{0.618 [0.591, 0.645]} \\
\bottomrule
\end{tabular}
}
\caption{Factual post-TACE CT generation on 79 paired OOF episodes. Brackets denote patient-bootstrap 95\% confidence intervals. Tumor HU-MAE is reported after inverse transformation to Hounsfield units. $^\ast$ Task-matched reimplementation of MeWM's action-conditioned
generative-dynamics principle, adapted to the two regimens and public
HCC-TACE-Seg outputs available in this study. All methods use the same
input state, preprocessing, and patient-level OOF folds.}
\label{tab:generation}
\end{table*}
\begin{table*}[t]
\centering

\newcommand{\estci}[3]{%
  \begin{tabular}[c]{@{}c@{}}
  #1\\
  {[#2, #3]}
  \end{tabular}%
}

{
\small
\begin{tabular}{@{}lccccc@{}}
\toprule
Planning method &
\begin{tabular}[c]{@{}c@{}}
Observed-regimen\\
agreement $\uparrow$
\end{tabular} &
\begin{tabular}[c]{@{}c@{}}
Macro-F1\\
$\uparrow$
\end{tabular} &
\begin{tabular}[c]{@{}c@{}}
Balanced accuracy\\
$\uparrow$
\end{tabular} &
\begin{tabular}[c]{@{}c@{}}
CaMeWM-only\\
support$^{\dagger}$ $\uparrow$
\end{tabular} &
\begin{tabular}[c]{@{}c@{}}
CaMeWM-only\\
abstention$^{\dagger}$ $\downarrow$
\end{tabular} \\
\midrule

CoxPH
& \estci{0.367}{0.266}{0.468}
& \estci{0.269}{0.210}{0.319}
& \estci{0.500}{0.500}{0.500}
& -- & -- \\

DeepSurv
& \estci{0.342}{0.241}{0.443}
& \estci{0.310}{0.223}{0.409}
& \estci{0.429}{0.337}{0.516}
& -- & -- \\

IPTW-Cox
& \estci{0.367}{0.266}{0.468}
& \estci{0.269}{0.210}{0.319}
& \estci{0.500}{0.500}{0.500}
& -- & -- \\

MeWM$^{\ast}$
& \estci{0.367}{0.266}{0.468}
& \estci{0.269}{0.210}{0.319}
& \estci{0.500}{0.500}{0.500}
& -- & -- \\

\textbf{CaMeWM}
& \estci{\textbf{0.506}}{0.392}{0.620}
& \estci{\textbf{0.503}}{0.388}{0.610}
& \estci{\textbf{0.523}}{0.401}{0.631}
& \estci{0.597}{0.561}{0.634}
& \estci{0.000}{0.000}{0.000} \\

\midrule
\multicolumn{6}{@{}l@{}}{%
\textit{CaMeWM factual-only AIPW:}
Observed RMST $=103.62$ weeks;\quad
AIPW policy value $=113.22$ [102.34, 124.84] weeks.} \\

\multicolumn{6}{@{}l@{}}{%
\textit{AIPW gain}
$=\mathbf{+9.60}$ [0.61, 18.13] weeks;\quad
Policy ESS $=37.4$;\quad
maximum correction weight $=2.66$ (cap $=10$).} \\

\bottomrule
\end{tabular}
}

\caption{Treatment-planning behavior and factual-only causal evaluation over
five patient-level out-of-fold folds. Main estimates are shown above
patient-bootstrap 95\% confidence intervals. Agreement metrics compare
recommendations with observed historical regimens and are behavioral
references, not counterfactual treatment truth. $^{\dagger}$ Defined only for
CaMeWM, which applies a propensity-support guardrail; ``--'' denotes not
applicable. $^{\ast}$ Task-matched reimplementation inspired by MeWM's
action-conditioned generative-dynamics principle.}
\label{tab:planning}
\end{table*}
\subsection{Comparators, Metrics, and Statistics}
We compare persistence, 3D U-Net, CVAE-3D, mask-DDPM, and MeWM generation; the latter retains action-conditioned diffusion but omits CaMeWM's response factor and causal adjustment.  Primary generation metrics are tumor HU-MAE and change-map $\Delta$SSIM, with global fidelity metrics reported in Table~\ref{tab:generation}.  CoxPH, DeepSurv, IPTW-Cox, and MeWM provide survival and planning comparators under the same folds.  We report Uno's C-index \cite{uno2011cindex}, IBS, RMST MAE, and Harrell's C-index for survival.  Historical-regimen agreement, macro-F1, and balanced accuracy are behavioral references, not counterfactual ground truth.

The primary policy estimate is factual-only cross-fitted AIPW-RMST, using $\eta=0.05$ support restriction and a correction-weight cap of 10; its outcome nuisance never uses generated states.  We report support, abstention, and policy ESS alongside the paired AIPW interval.  Propensity balance, generated-to-factual survival transport, $K$-world stability, and pseudo-mask audits are diagnostic analyses in Supplementary Figs.~S1--S4.

\begin{table*}[t]
\centering
{\small
\begin{tabular}{lcccc}
\toprule
Variant & Normalized tumor MAE $\downarrow$ & SSIM $\uparrow$ & Global HU-MAE $\downarrow$ & $\Delta$SSIM $\uparrow$ \\
\midrule
\textbf{Full CaMeWM} & 0.070 [0.059, 0.082] & 0.618 [0.591, 0.644] & 61.102 [54.384, 68.512] & 0.528 [0.483, 0.570] \\
Causal anchor only & 0.070 [0.060, 0.081] & 0.612 [0.585, 0.637] & 61.139 [54.508, 68.499] & 0.526 [0.482, 0.569] \\
w/o response-code conditioning & 0.071 [0.060, 0.082] & 0.618 [0.591, 0.643] & 61.198 [54.274, 68.508] & 0.528 [0.485, 0.571] \\
w/o direct action path & 0.075 [0.064, 0.087] & 0.617 [0.591, 0.644] & 62.239 [55.086, 69.746] & 0.524 [0.480, 0.566] \\
w/o confidence weighting & 0.070 [0.060, 0.082] & 0.618 [0.592, 0.644] & 61.197 [54.610, 68.239] & 0.527 [0.482, 0.570] \\
\bottomrule
\end{tabular}
}
\caption{Controlled generation ablation over the same five patient-level out-of-fold folds. Normalized tumor MAE is computed on the model's $[-1,1]$ CT scale and is not numerically comparable to the HU-MAE in Table~\ref{tab:generation}. ``Deterministic anchor only'' removes residual-DDPM refinement; ``uniform pseudo-mask weighting'' sets all pseudo-mask confidence weights to one. Brackets denote patient-bootstrap 95\% confidence intervals.}
\label{tab:ablation}
\end{table*}

\subsection{Results}
\paragraph{Factual transition fidelity.}
Table~\ref{tab:generation} evaluates factual post-TACE transition prediction on the 79 paired OOF episodes.  CaMeWM has the strongest structural scores by point estimate ($\Delta$SSIM $0.528$ and SSIM $0.618$), while its tumor HU-MAE ($27.145$) is essentially tied with the deterministic 3D U-Net ($26.944$) within bootstrap uncertainty.  It improves on CVAE-3D and the action-conditioned diffusion baselines in all four reported metrics.  Thus the evidence supports a structural-fidelity advantage and competitive local intensity fidelity.
\paragraph{Survival utility on factual and generated states.}
\begin{table}[t]
\centering
{\small
\setlength{\tabcolsep}{5.5pt}
\begin{tabular}{@{}lcccc@{}}
\toprule
Method & Uno C $\uparrow$ & IBS $\downarrow$ & \shortstack{RMST\\MAE $\downarrow$} & \shortstack{C-index $\uparrow$\\95\% CI} \\
\midrule
CoxPH & 0.580 & 0.202 & 39.28 & \shortstack{0.581\\ \mbox{[0.500, 0.661]}} \\
DeepSurv & 0.592 & \textbf{0.200} & 39.17 & \shortstack{0.597\\\mbox{[0.525, 0.670]}} \\
IPTW-Cox & 0.603 & \textbf{0.200} & \textbf{38.73} & \shortstack{0.609\\\mbox{[0.531, 0.689]}} \\
MeWM$^\ast$  & 0.461 & 0.242 & 44.53 & \shortstack{0.462\\\mbox{[0.390, 0.540]}} \\
\shortstack[l]{CaMeWM factual} & 0.589 & 0.211 & 40.80 & \shortstack{0.592\\\mbox{[0.521, 0.662]}} \\
\shortstack[l]{\textbf{CaMeWM} \\ generated worlds} & \textbf{0.612} & \underline{0.202} & \underline{39.23} & \shortstack{\textbf{0.615}\\\mbox{[0.546, 0.688]}} \\
\bottomrule
\end{tabular}
}
\caption{Patient-level out-of-fold survival prediction over five folds. The generated-world row evaluates the complete action-conditioned transition-and-utility path. Brackets denote patient-bootstrap 95\% confidence intervals. $^\ast$ Task-matched reimplementation inspired by MeWM's action-conditioned generative-dynamics principle.}
\label{tab:survival}
\end{table}

Table~\ref{tab:survival} separates a diagnostic that consumes the observed follow-up state from the planning-time path that consumes eight generated worlds.  The factual-state response-Cox utility model reaches C-index $0.592$ (95\% CI $[0.521,0.662]$).  Replacing the factual state with generated worlds yields C-index $0.615$ ($[0.546,0.688]$), Uno C $0.612$, IBS $0.202$, and RMST MAE $39.23$ weeks.  This end-to-end result is comparable to the strongest tabular baseline, IPTW-Cox (C-index $0.609$), but their intervals overlap; we therefore report competitive rather than superior survival discrimination.
\begin{figure}[t]
\centering
\includegraphics[width=0.40\textwidth]{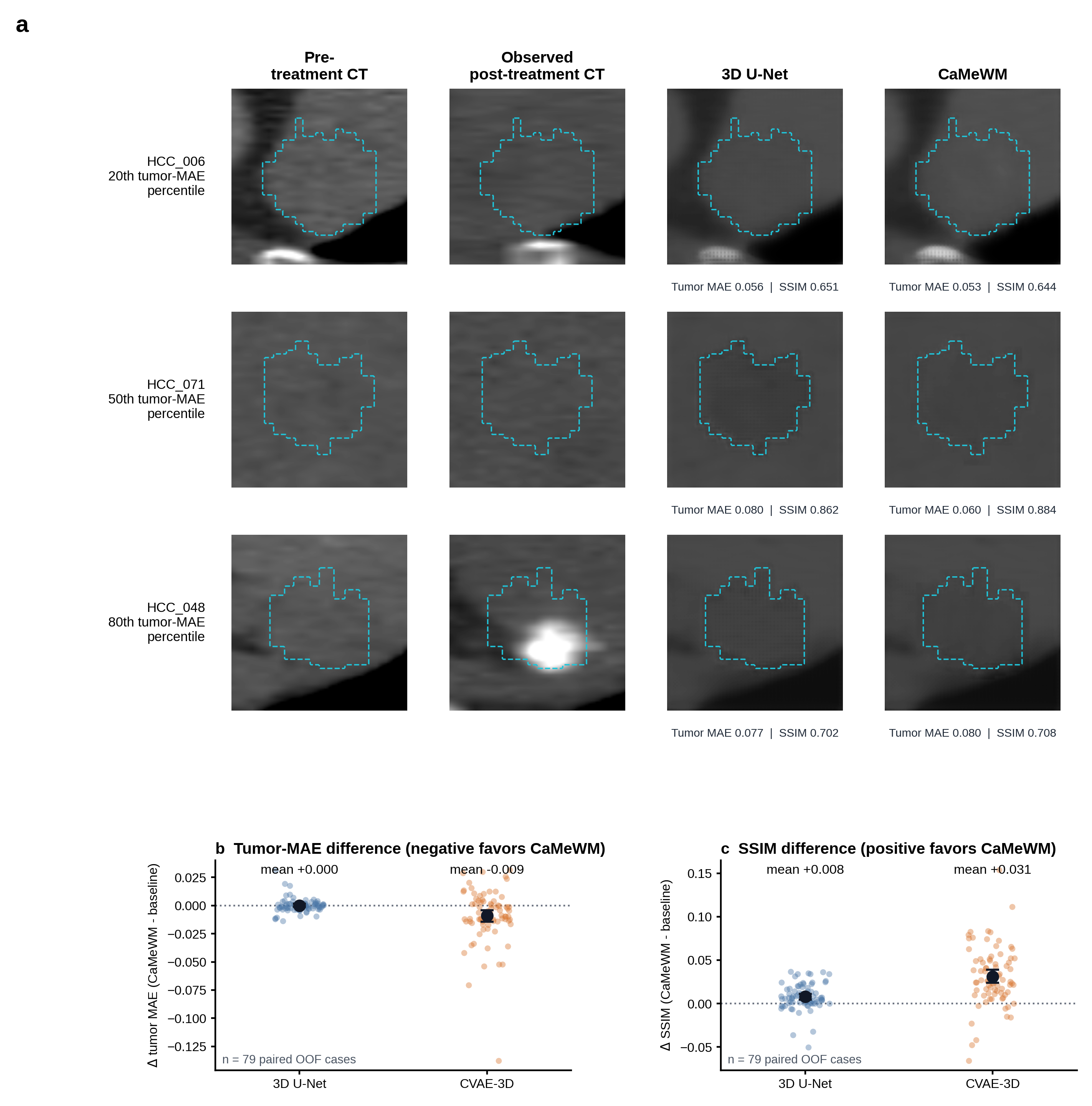}
\caption{Out-of-fold transition fidelity and paired comparison. CaMeWM synthesizes action-conditioned follow-up CT volumes in registered
image space, rather than tumor masks. Each row shows pre-treatment CT,
observed post-treatment CT, and predictions from 3D U-Net and the mean
of eight CaMeWM worlds. Cyan dashed contours mark the pre-treatment
lesion only for orientation; they are neither follow-up annotations nor
model outputs. Tumor-region MAE uses a fold-local pseudo post-treatment
mask, whereas SSIM is computed over the complete follow-up CT. Rows are
deterministically selected at the 20th, 50th, and 80th MAE percentiles.}
\label{fig:qualitative}
\end{figure}
\paragraph{Planning behavior and causal diagnostics.}
Table~\ref{tab:planning} reports observed-regimen agreement, macro-F1, and
balanced accuracy as behavioral references rather than counterfactual truth.
CaMeWM obtains $0.506$ agreement and $0.503$ macro-F1, versus $0.342$--$0.367$
agreement for the baselines. Its recommendations have mean estimated support
$0.597$ and no abstentions at the prespecified threshold. Under factual-only,
cross-fitted AIPW evaluation, CaMeWM yields an estimated RMST gain of
$9.60$ weeks [0.61, 18.13], with policy ESS $37.4$. This internal
observational estimate does not establish prospective clinical benefit.

\paragraph{Controlled generation ablation.}
Table~\ref{tab:ablation} isolates generation components under the same five
folds. Removing the direct action path produces the largest degradation
(tumor MAE $0.075$ versus $0.070$; HU-MAE $62.239$ versus $61.102$),
supporting a direct treatment-to-transition route. Other changes are small
with overlapping intervals and are interpreted as sensitivity evidence. Supplementary Fig.~S5 uses a semi-synthetic structural model with known
effects to assess ATE and policy recovery, DR identity, positivity, factor
error, Monte Carlo stability, and hidden-confounding sensitivity. These
controlled checks validate estimator behavior, not clinical benefit.

\section{Conclusion and Limitations}
CaMeWM integrates action-conditioned CT simulation, support-restricted treatment comparison, and cross-fitted factual-only AIPW evaluation, yielding competitive structural fidelity, stable recommendations, and a positive policy-value estimate. Limitations include the small retrospective cohort, only two supported regimens, partition-dependent uncertainty, and no external or prospective validation. Identification relies on measured exchangeability, positivity, and correct censoring adjustment; marginal IPCW may be biased under informative censoring. Predicted masks, unvalidated response descriptors, deterministic uncertainty, and a single transition further limit dynamic-policy inference. CaMeWM is a research framework, not an autonomous clinical system.

\clearpage
\onecolumn
\phantomsection
\pdfbookmark[0]{Supplementary Material}{supplementary-material}
\begin{center}
{\LARGE\bfseries Supplementary Material}\par\medskip
{\large CaMeWM: A Causal Medical World Model for Generative Tumor Evolution and Treatment Planning}
\end{center}
\setcounter{secnumdepth}{2}
\setcounter{section}{0}
\setcounter{subsection}{0}
\setcounter{equation}{0}
\setcounter{table}{0}
\setcounter{figure}{0}
\setcounter{theorem}{0}
\setcounter{assumption}{0}
\setcounter{definition}{0}
\setcounter{remark}{0}
\renewcommand{\thesection}{S\arabic{section}}
\renewcommand{\thesubsection}{\thesection.\arabic{subsection}}
\renewcommand{\theequation}{S\arabic{equation}}
\renewcommand{\thetable}{S\arabic{table}}
\renewcommand{\thefigure}{S\arabic{figure}}
\renewcommand{\thetheorem}{S\arabic{theorem}}
\renewcommand{\theassumption}{S\arabic{assumption}}
\renewcommand{\thedefinition}{S\arabic{definition}}
\renewcommand{\theremark}{S\arabic{remark}}
\renewcommand{\theHsection}{supp.\arabic{section}}
\renewcommand{\theHsubsection}{supp.\arabic{section}.\arabic{subsection}}
\renewcommand{\theHequation}{supp.\arabic{equation}}
\renewcommand{\theHtable}{supp.\arabic{table}}
\renewcommand{\theHfigure}{supp.\arabic{figure}}
\renewcommand{\theHtheorem}{supp.\arabic{theorem}}
\renewcommand{\theHassumption}{supp.\arabic{assumption}}
\renewcommand{\theHdefinition}{supp.\arabic{definition}}
\renewcommand{\theHremark}{supp.\arabic{remark}}

\subsection*{Overview}
We give an implementation-aligned causal foundation for a medical world model
that simulates post-treatment tumor states and ranks supported TACE regimens.
An explicit structural causal model identifies supported-regimen restricted
mean survival time (RMST) through a response-factorized g-formula.  A
factor-wise total-variation bound propagates response-proxy, image-transition,
and survival-model errors to treatment-value error and policy regret.
Cross-fitted factual-only AIPW evaluation is consistent when either the
joint-regimen propensity or a separately fitted factual clinical outcome
regression is correct, subject to the stated treatment and censoring
assumptions.  A Monte Carlo result quantifies
the sampling error incurred when integrating over generated follow-up states.
The current implementation uses a deterministic response-proxy predictor, a
residual conditional DDPM, a clipped joint-IPTW weighted response-Cox survival
model, and a support-restricted Monte Carlo lower-confidence score.  We state
separately the stronger simultaneous calibration condition that would be
needed for a finite-sample LCB near-optimality guarantee.  Finally, we show why
neither unmeasured-confounding bias nor a patient's unique cross-world CT image
is identified from these observational data.

\section{Scope of the Theoretical Claim}

MeWM models treatment-conditioned tumor evolution and uses a survival model to
score simulated states \cite{yang2025mewm}.  CaMeWM changes the statistical
object being estimated.  Rather than interpreting an action-conditioned image
generator as causal by construction, CaMeWM defines a treatment value under an
explicit structural causal model (SCM), states the assumptions under which that
value is identified, and evaluates the frozen resulting policy with an
orthogonal factual-only AIPW estimator.  Recent AAAI causal learning papers similarly connect
an SCM to an estimand, an estimator, and a testable empirical consequence.  The
relevant precedents cover patient trajectories \cite{frauen2023deepace}, causal
representations \cite{li2024cfi}, dynamic interventions
\cite{cinquini2025causal}, and treatment-effect uncertainty
\cite{horii2024uncertainty}.

The theory identifies \emph{interventional distributions and expected treatment
values on the observed support}.  It does not identify an unobserved patient's
unique alternate-world CT image, and it does not remove bias from unmeasured
confounding.  These boundaries are formalized in
Proposition~\ref{prop:nonidentification} and
Remark~\ref{rem:crossworld}.

This supplementary material accompanies the main text of this preprint
and is aligned to its fixed execution path:
joint action encoding, two curated regimens, fold-specific preprocessing,
deterministic response prediction, residual diffusion, response-Cox survival,
fixed five-fold patient-level out-of-fold evaluation, and eight Monte
Carlo worlds per candidate regimen.  Population theorems are distinguished from engineering
approximations.  In particular, the real-data Monte Carlo LCB is not described
as a calibrated simultaneous confidence bound.

Throughout this supplement, a \emph{factual state} is the observed registered
follow-up CT, whereas a \emph{generated world} is a sampled follow-up CT from
the fitted transition model.  The generated-world survival row in the main
paper is an end-to-end transition-and-utility diagnostic using $K=8$ worlds; it
is not a causal policy-value estimate.  The primary policy-value result instead
uses the factual-only, cross-fitted AIPW estimator in
Section~\ref{sec:factual-aipw}, whose outcome nuisance receives only
pre-treatment clinical history, factual regimen, and observed survival.

\subsection{Main-paper correspondence and reading map}

The supplement is organized around the main paper's central claims.
Table~\ref{tab:claim-map} makes the correspondence explicit so that no
theoretical statement is detached from its implementation or empirical test.

\begin{table}[ht]
\centering
\small
\caption{Claim-to-evidence map.  ``Required evidence'' denotes evidence needed
to activate the corresponding main-paper claim, not optional analysis.}
\label{tab:claim-map}
\begin{tabular}{p{0.25\linewidth}p{0.28\linewidth}p{0.39\linewidth}}
\toprule
Main-paper contribution & Formal support & Required implementation and evidence \\
\midrule
Support-restricted interventional world distribution &
Theorem~\ref{thm:identification} and its complete proof &
Joint-regimen propensity, explicit support set, response-factorized g-formula,
overlap diagnostics, and semi-synthetic identification recovery. \\
Response-factorized residual transition &
Theorem~\ref{thm:value-error} and Corollary~\ref{cor:plugin-regret} &
Operational response proxies, lesion-local residual DDPM, controlled component
ablations, factor-wise errors, and empirical bound coverage. \\
RMST world integration and supported policy &
Lemma~\ref{lem:mc} and Theorem~\ref{thm:lcb} &
Response-Cox RMST, $K=8$ worlds per action, sampling-convergence analysis, and
support abstention.  The current score is not claimed to be a calibrated
finite-sample confidence bound. \\
Cross-fitted factual policy evaluation &
Theorem~\ref{thm:factual-aipw}, its proof, and the censoring identity &
Outer-fold predictions, IPCW RMST pseudo-outcomes, factual-only AIPW value,
overlap/balance/ESS diagnostics, and patient-level bootstrap intervals. \\
Claim boundaries &
Proposition~\ref{prop:nonidentification} and Remark~\ref{rem:crossworld} &
Hidden-confounding sensitivity and explicit exclusion of unique individual CT
and prospective-benefit claims. \\
\bottomrule
\end{tabular}
\end{table}

\section{Setup and Causal Medical World Model}

\subsection{Observed data and treatment value}

For patient $i$, the observed record is
\begin{equation}
 O_i=(H_i,A_i,M_i,X_i^+,\widetilde T_i,\Delta_i),
 \qquad i=1,\ldots,n,
 \label{eq:observed-data}
\end{equation}
where $H=(X^-,S^-,C)$ combines the registered pre-treatment CT $X^-$,
pre-treatment tumor mask $S^-$, and clinical covariates $C$; $A\in\Aset$ is a
complete TACE regimen from a finite candidate library; $M\in\R^{5}$ is an
operational treatment-response descriptor; $X^+=(I^+,S^+)$ denotes the
registered follow-up CT and tumor-mask state; $T$ is event time;
$\widetilde T=\min(T,C^{\mathrm{cen}})$; and
$\Delta=\1\{T\le C^{\mathrm{cen}}\}$.  Importantly, $A$ is a \emph{joint
regimen}, not a set of independent Bernoulli drug labels.  The theory therefore
requires the generalized propensity
\begin{equation}
 e_a(h)=\Pp(A=a\mid H=h), \qquad a\in\Aset.
 \label{eq:propensity}
\end{equation}

In the current HCC-TACE-Seg pipeline, $S^+$ is unavailable as a manual label.
A segmenter trained within each outer fold on pre-treatment CT--mask pairs is
applied to follow-up CT, with a deterministic intensity-based fallback for
small predicted masks.  Let $n^-$ and $n^+$ be the pre/post mask voxel counts,
and let $z^-=(z^-_j)$ and $z^+=(z^+_j)$ be their masked HU values.  The extractor
implements
\begin{equation}
 M=\left(
 \frac{n^+-n^-}{\max(n^-,1)},
 \frac{1}{|z^+|}\sum_j\1\{z_j^+>150\},
 \frac{1}{|z^+|}\sum_j\1\{z_j^+<30\},
 \overline z^- - \overline z^+,
 \frac{1}{|z^+|}\sum_j\1\{30\le z_j^+\le150\}
 \right).
 \label{eq:operational-response}
\end{equation}
The components are relative volume change, lipiodol deposition, necrosis proxy,
enhancement reduction, and residual viable-tumor proxy.  The fixed HU cutoffs
make the extractor reproducible but protocol-dependent; mask perturbation is
therefore analyzed explicitly.  Thus $M$ is a pseudo-labeled, image-derived
response descriptor, not a separately timed biological mediator, a validated
radiomic biomarker, or evidence for a natural direct or indirect treatment
effect.

For a clinically fixed horizon $\tau>0$, define normalized restricted mean
survival utility
\begin{equation}
 U_\tau(a)
 =\frac{\min\{T(a),\tau\}}{\tau}
 =\frac{1}{\tau}\int_0^\tau \1\{T(a)>t\}\,dt
 \in[0,1].
 \label{eq:rmst}
\end{equation}
RMST is used because it has a direct time-based interpretation and does not
require proportional hazards \cite{royston2013rmst}.  The individual conditional
value and population value are
\begin{equation}
 V_\tau(a,h)=\E\{U_\tau(a)\mid H=h\},
 \qquad
 V_\tau(a)=\E_H\{V_\tau(a,H)\}.
 \label{eq:value}
\end{equation}

The mathematical analysis uses normalized utility in $[0,1]$.  Code and CSV
outputs store RMST in weeks, $R_\tau(a)=\tau U_\tau(a)$.  Dividing reported
values and errors by $\tau$ recovers the normalized scale required by the
bounded-loss and Hoeffding results; this rescaling does not change regimen
rankings.

\subsection{Structural equations and causal graph}

CaMeWM uses the recursive SCM
\begin{align}
 H &= f_H(U_H),
 & A &= f_A(H,U_A), \nonumber\\
 M &= f_M(H,A,U_M),
 & X^+ &= f_X(H,A,M,U_X), \label{eq:scm}\\
 T &= f_T(H,A,M,X^+,U_T). \nonumber
\end{align}
The corresponding DAG contains
\begin{equation*}
 H\to\{A,M,X^+,T\},\quad
 A\to\{M,X^+,T\},\quad
 M\to\{X^+,T\},\quad X^+\to T.
\end{equation*}
The direct edge $A\to T$ is retained to avoid forcing all regimen effects
through image-derived response descriptors.  Under intervention $\doop(A=a)$, the SCM
induces
\begin{equation}
 p(m,x,t\mid h,\doop(a))
 =p_M(m\mid h,a)\,p_X(x\mid h,a,m)\,p_T(t\mid h,a,m,x).
 \label{eq:causal-factorization}
\end{equation}
This response-factorized transition is CaMeWM's central modeling structure.
Because the descriptor and follow-up image are obtained at the same visit, the
ordering $M$ then $X^+$ is a hierarchical factorization of their joint
post-treatment law.  It must not be interpreted as observed biological time
precedence or as identification of path-specific mediation effects.

\section{Identification}

\subsection{Assumptions}

\begin{assumption}[Consistency, SUTVA, and no interference]
\label{ass:consistency}
If $A=a$, consistency requires
\begin{equation*}
 (M,X^+,T)=(M(a),X^+(a),T(a)).
\end{equation*}
One patient's treatment does not alter another patient's trajectory.
\end{assumption}

\begin{assumption}[Conditional treatment exchangeability]
\label{ass:exchangeability}
For every $a\in\Aset$,
\begin{equation}
 A\ind\{M(a),X^+(a),U_\tau(a)\}\mid H.
 \label{eq:exchangeability}
\end{equation}
Thus $H$ contains all common causes of regimen assignment and subsequent
response.  This is an untestable identifying assumption, not a consequence of
neural-network training.
\end{assumption}

\begin{assumption}[Treatment positivity]
\label{ass:positivity}
For histories in the target population and every regimen considered by the
policy, $e_a(h)\ge\eta>0$.
\end{assumption}

\begin{assumption}[Pre-treatment adjustment and post-treatment factorization]
\label{ass:temporal}
$H$ is measured before $A$; $M$ and $X^+$ are post-treatment variables; and no
descendant of $A$ is included in the propensity model.  The chain-rule ordering
$p_M(m\mid h,a)p_X(x\mid h,a,m)$ is an estimation factorization of the joint
post-treatment response and does not require $M$ to be measured before $X^+$.
\end{assumption}

\begin{assumption}[Independent censoring]
\label{ass:censoring}
Let $Z=(H,A,M,X^+)$.  Then
$C^{\mathrm{cen}}\ind(T,Z)$ and
$G(t)=\Pp(C^{\mathrm{cen}}\ge t)\ge g_{\min}>0$ for
$0\le t\le\tau$.
This stronger marginal condition matches the Kaplan--Meier censoring weights
used in the released evaluation code.  A conditional censoring model
$G(t\mid Z)$ is a valid extension but is not implemented here.
\end{assumption}

\subsection{Causal g-formula}

Define the conditional outcome regression
\begin{equation}
 Q_\tau(h,a,m,x)
 =\E(U_\tau\mid H=h,A=a,M=m,X^+=x).
 \label{eq:q-function}
\end{equation}

This $Q_\tau$ is the planning-time generative functional: it maps a supported
history and regimen through the response and follow-up-state factors.  It is
not the factual clinical outcome nuisance used to evaluate the frozen policy
in Section~\ref{sec:factual-aipw}.

\begin{theorem}[Identification by the response-factorized world-model g-formula]
\label{thm:identification}
Under Assumptions~\ref{ass:consistency}--\ref{ass:temporal}, for every supported
$(a,h)$,
\begin{equation}
 V_\tau(a,h)
 =\int\!\!\int
 Q_\tau(h,a,m,x)
 p_X(x\mid h,a,m)
 p_M(m\mid h,a)\,dx\,dm.
 \label{eq:g-formula}
\end{equation}
Consequently,
\begin{equation}
 \pi^*(h)\in\arg\max_{a\in\Asup(h)}V_\tau(a,h),
 \qquad
 \Asup(h)=\{a\in\Aset:e_a(h)\ge\eta\},
 \label{eq:oracle-policy}
\end{equation}
is optimal among regimens satisfying positivity.
\end{theorem}

Theorem~\ref{thm:identification} is a back-door/g-computation result
\cite{pearl2009causality,hernan2020whatif}.  The response descriptor supplies a
clinically interpretable hierarchical factorization of the state transition;
it is not used to assert front-door or path-specific mediation identification.

\section{Implementation-Aligned Staged Estimation}

CaMeWM estimates the three g-formula factors and the treatment mechanism:
\begin{equation}
 \widehat P_M(m\mid h,a),\qquad
 \widehat P_X(x\mid h,a,m),\qquad
 \widehat Q_\tau(h,a,m,x),\qquad
 \widehat e_a(h).
 \label{eq:three-factors}
\end{equation}
The released implementation trains these modules sequentially within each
outer training fold.  It does not optimize a single end-to-end loss, an MMD
balance penalty, or a mediator--state consistency loss.

\paragraph{Auxiliary response-descriptor mechanism.}
Let $z_M$ denote the component-wise transformed and training-fold standardized
version of the five-dimensional descriptor $M$.  The 3D network $f_\phi$ receives
$(X^-,S^-,C,A)$ and minimizes
\begin{equation}
 \mathcal L_M(\phi)
 =\E\!\left[\operatorname{SmoothL1}
 \{f_\phi(X^-,S^-,C,A),z_M\}\right].
 \label{eq:mediator-loss}
\end{equation}
The checkpoint is selected by validation standardized MAE.  At inference,
$f_\phi$ is inverse-standardized and range-clamped.  The implemented descriptor
law is therefore the point mass
\begin{equation}
 \widehat P_M(\cdot\mid h,a)=\delta_{\widehat m(h,a)},
 \qquad \widehat m(h,a)=f_\phi(h,a),
 \label{eq:point-mediator}
\end{equation}
rather than a learned stochastic descriptor distribution.  This auxiliary
factor is not used to claim causal mediation, front-door identification, or a
path-specific treatment effect.

\paragraph{Deterministic anchor and residual conditional diffusion transition.}
CT intensities are windowed to $[-1,1]$.  Given $(H,A,M)$, the deterministic
anchor $B_\beta(H,A,M)$ predicts a coarse treatment-conditioned follow-up CT.
The DDPM models only the normalized residual $D=(I^+-B_\beta)/2$.  At diffusion
step $s$,
\begin{equation}
 D_s=\sqrt{\bar\alpha_s}D
     +\sqrt{1-\bar\alpha_s}\,\epsilon,
 \qquad \epsilon\sim\mathcal N(0,I),
\end{equation}
and the denoiser conditions on $(D_s,s,B_\beta,S^-,A,M)$.  If
$\widetilde I^+_\theta=B_\beta+2\widetilde D_\theta$ is the reconstructed
follow-up CT and $W$ is the confidence-weighted union of the baseline and
pseudo follow-up masks defined in the main paper, the implemented loss is
\begin{align}
 \mathcal L_X(\theta)
 &=\E\norm{\epsilon-\epsilon_\theta(D_s,s,B_\beta,S^-,A,M)}_2^2
   +\lambda_g\norm{\widetilde I^+_\theta-I^+}_1 \nonumber\\
 &\quad+\lambda_t
 \frac{\norm{W\odot(\widetilde I^+_\theta-I^+)}_1}
      {\norm{W}_1+10^{-6}}
 +\lambda_s\{1-\operatorname{SSIM}_3(\widetilde I^+_\theta,I^+)\},
 \label{eq:diffusion-loss}
\end{align}
with $(\lambda_g,\lambda_t,\lambda_s)=(0.5,2.0,0.25)$.  Residual changes are
restricted to a dilated pre-treatment tumor region during decoding.  To reduce
the factual-to-predicted descriptor shift, each training batch uses the
predicted descriptor with probability $0.5$ and the factual descriptor
otherwise; validation and planning use the predicted descriptor.  The diffusion
checkpoint is selected by validation tumor-region MAE.
During planning, the fold-trained segmenter is applied to each generated CT to
obtain the generated post-mask supplied to the survival network.

\paragraph{Joint-regimen propensity.}
For the two mutually exclusive curated regimens, a baseline 3D encoder and
clinical head estimate
\begin{equation}
 \widehat e_a(h)=\operatorname{softmax}_a r_\gamma(X^-,S^-,C),
 \qquad
 \mathcal L_e(\gamma)=-\E\log\widehat e_A(H).
 \label{eq:propensity-loss}
\end{equation}
No post-treatment image, mask, or response proxy enters this model.  The
checkpoint is selected by validation cross-entropy.  Overlap is assessed after
training using selected propensities, effective sample size, clipping rate, and
pre/post-weighting standardized mean differences \cite{austin2015iptw}.

\paragraph{Joint-IPTW response-Cox survival.}
The regularized response-Cox head conditions on
$Z=(H,A,M,X^+)$ through fold-local response features and is trained using the
clipped stabilized joint-regimen weight
\begin{equation}
 w_i=\frac{\operatorname{clip}\!\left\{
 \Pr(A=A_i)/\widehat e_{A_i}(H_i);\,1/10,10\right\}}
 {\frac{1}{n}\sum_{j=1}^n\operatorname{clip}\!\left\{
 \Pr(A=A_j)/\widehat e_{A_j}(H_j);\,1/10,10\right\}}
 \label{eq:implemented-weight}
\end{equation}
and the weighted Cox partial negative log-likelihood
\begin{equation}
 \mathcal L_T(\omega)
 =-\sum_{i:\Delta_i=1}w_i
 \left[r_\omega(Z_i)-
 \log\left\{\sum_{j:\widetilde T_j\ge\widetilde T_i}
 w_j\exp r_\omega(Z_j)\right\}\right]
 +\frac{\lambda}{2}\lVert\omega\rVert_2^2.
 \label{eq:survival-loss}
\end{equation}
At each outer fold, a weighted Breslow baseline hazard is fitted on the
training partition.  The resulting survival curve and RMST are
\begin{equation}
 \widehat S_\omega(t\mid Z)
 =\exp\left\{-\widehat H_0(t)\exp r_\omega(Z)\right\},\qquad
 \widehat R_\tau(Z)
 =\int_0^\tau \widehat S_\omega(t\mid Z)\,dt,
 \quad \widehat Q_\tau=\widehat R_\tau/\tau.
 \label{eq:cox-rmst}
\end{equation}
The primary checkpoint is selected by validation C-index, with validation loss
as the fallback when C-index is undefined.  Consequently, the released model is
a staged estimator with a weighted outcome loss, not the previously described
joint MMD-regularized objective.

Training uses factual $M_i$ and $X_i^+$, whereas planning evaluates
$\widehat Q_\tau\{h,a,\widehat m(h,a),\widehat X^+(a)\}$.  This substitution is a
generated-state covariate shift.  The implementation deliberately does not
augment survival training with model-generated images because doing so would
couple the outcome labels to the current generator and could create
self-confirming errors.  Consequently, factual OOF discrimination and
calibration evaluate the outcome module on observed states but do not certify
calibration on unobserved counterfactual states.  This gap is part of the model
error $\varepsilon_Q+\varepsilon_X$, not removed by propensity adjustment.

\paragraph{Notation reference.}
Table~\ref{tab:notation} collects the principal symbols used in the main paper
and this supplement.  Hats denote fitted quantities and $a$ denotes a candidate
regimen, whereas $A$ denotes the observed regimen.

\begin{table}[htbp]
\centering
\small
\caption{Principal notation used by CaMeWM.}
\label{tab:notation}
\begin{tabular}{p{0.17\linewidth}p{0.72\linewidth}}
\toprule
Symbol & Meaning \\
\midrule
$H=(I^-,S^-,C)$ & Pre-treatment history: baseline CT, baseline lesion mask, and clinical covariates. \\
$A$, $a$ & Observed joint regimen and a candidate intervention, respectively. \\
$M$ & Treatment-conditioned response descriptor predicted from $(H,a)$. \\
$X^+$, $X^+(a)$ & Observed registered follow-up CT and its potential follow-up state under $a$. \\
$e_a(H)$, $\mathcal A_\eta(H)$ & Joint-regimen propensity and the support-restricted candidate set satisfying $e_a(H)\geq\eta$. \\
$P_M$, $P_X$ & Response-descriptor and follow-up-state transition distributions in the g-formula. \\
$Y_\tau$, $Q_\tau$, $V_\tau$ & Restricted survival outcome, conditional RMST utility, and intervention value through horizon $\tau$. \\
$\pi$, $K$ & Supported treatment policy and number of sampled follow-up worlds per regimen. \\
\bottomrule
\end{tabular}
\end{table}

\section{Factor-Wise Error Propagation and Planning Regret}

For fixed $(h,a)$, define the true and estimated transition measures
\begin{align}
 P_{h,a}(dm,dx)
 &=P_M(dm\mid h,a)P_X(dx\mid h,a,m),\\
 \widehat P_{h,a}(dm,dx)
 &=\widehat P_M(dm\mid h,a)\widehat P_X(dx\mid h,a,m).
\end{align}
We use
$\TV(P,Q)=\sup_B|P(B)-Q(B)|$.

\begin{theorem}[Factor-wise causal-value error]
\label{thm:value-error}
Assume $Q_\tau,\widehat Q_\tau\in[0,1]$ and define
\begin{align}
 \varepsilon_Q(h,a)
 &=\sup_{m,x}|\widehat Q_\tau(h,a,m,x)-Q_\tau(h,a,m,x)|,\\
 \varepsilon_M(h,a)
 &=\TV\{\widehat P_M(\cdot\mid h,a),P_M(\cdot\mid h,a)\},\\
 \varepsilon_X(h,a)
 &=\sup_m\TV\{\widehat P_X(\cdot\mid h,a,m),
                       P_X(\cdot\mid h,a,m)\}.
\end{align}
Let $\widehat V_\tau(a,h)$ be Eq.~\eqref{eq:g-formula} with all three factors
replaced by their estimates.  Then
\begin{equation}
 |\widehat V_\tau(a,h)-V_\tau(a,h)|
 \le \varepsilon_Q(h,a)+\varepsilon_M(h,a)+\varepsilon_X(h,a).
 \label{eq:value-bound}
\end{equation}
\end{theorem}

\begin{corollary}[Plug-in policy regret]
\label{cor:plugin-regret}
Let
$\widehat\pi(h)\in\arg\max_{a\in\Asup(h)}\widehat V_\tau(a,h)$ and
\begin{equation*}
 \varepsilon(h)=\max_{a\in\Asup(h)}
 \{\varepsilon_Q(h,a)+\varepsilon_M(h,a)+\varepsilon_X(h,a)\}.
\end{equation*}
Then
\begin{equation}
 V_\tau(\pi^*(h),h)-V_\tau(\widehat\pi(h),h)
 \le 2\varepsilon(h).
 \label{eq:plugin-regret}
\end{equation}
\end{corollary}

Equation~\eqref{eq:plugin-regret} is the reason the paper must report response-descriptor,
generation, and survival errors separately.  A visually realistic transition
does not imply small planning regret when either descriptor or outcome error is
large.

\begin{remark}[Scope of the factor bound for the released model]
Theorem~\ref{thm:value-error} is valid for the generic factorized law.  In the
real-data implementation, Eq.~\eqref{eq:point-mediator} replaces the descriptor
law by a point mass and no observable counterfactual distribution is available
to estimate the three total-variation errors.  The theorem is therefore tested
with known ground truth in the bounded semi-synthetic experiment; real-data
response-descriptor MAE, generation error, and survival error are diagnostics rather than
empirical estimates of $\varepsilon_M$, $\varepsilon_X$, and
$\varepsilon_Q$.
\end{remark}

Accordingly, the bound is a structural decomposition rather than a computable
finite-sample certificate for HCC-TACE-Seg.  Standardized response MAE,
tumor-region HU-MAE and change-map SSIM, and survival IBS/calibration are
observable module diagnostics only; none is asserted to equal a uniform
total-variation term.  Exact factor errors and bound coverage are reported only
in the bounded semi-synthetic SCM where both factual and counterfactual laws are
known.

\subsection{Monte Carlo world-model integration}

In the released planner, Eq.~\eqref{eq:g-formula} is approximated with one
deterministic $\widehat m(h,a)$ and $K$ conditionally independent diffusion
samples $\widehat X_k^+$.  Define
\begin{equation}
 \widehat V_{\tau,K}(a,h)
 =\frac1K\sum_{k=1}^K
 \widehat Q_\tau(h,a,\widehat m(h,a),\widehat X_k^+).
 \label{eq:mc-value}
\end{equation}
The AAAI configuration uses $K=8$.  A future stochastic descriptor model could
sample $(\widehat M_k,\widehat X_k^+)$ jointly without changing the argument.

\begin{lemma}[Finite-sample Monte Carlo error]
\label{lem:mc}
Conditionally on the fitted world model,
\begin{equation}
 \Pp\left(
 |\widehat V_{\tau,K}(a,h)-\widehat V_\tau(a,h)|\ge r
 \right)\le 2e^{-2Kr^2}.
 \label{eq:mc-hoeffding}
\end{equation}
If $J=|\Asup(h)|$, then with probability at least $1-\delta$ the bound holds
simultaneously for all supported actions with
\begin{equation}
 r_K(\delta,J)=\sqrt{\frac{\log(2J/\delta)}{2K}}.
 \label{eq:mc-radius}
\end{equation}
\end{lemma}

The lemma is stated for normalized RMST.  For the day-scale values written by
the code, the deviation radius is $\tau r_K(\delta,J)$.

\begin{corollary}[End-to-end sampled-policy regret]
\label{cor:sampled-regret}
Let
$\widehat\pi_K(h)\in\arg\max_{a\in\Asup(h)}\widehat V_{\tau,K}(a,h)$.
With probability at least $1-\delta$ over world-model sampling,
\begin{equation}
 V_\tau(\pi^*(h),h)-V_\tau(\widehat\pi_K(h),h)
 \le 2\{\varepsilon(h)+r_K(\delta,J)\}.
 \label{eq:sampled-regret}
\end{equation}
\end{corollary}

\section{Operational Response Semantics and Limitations}

The diffusion model is conditioned on the predicted response descriptor, but the
released code does not train a separate extractor that maps a generated CT back
to $M$, and it does not optimize a generated-image/descriptor consistency loss.
Consequently, descriptor prediction accuracy and image-generation accuracy are
reported separately.  Conditioning alone does not establish that generated CT
semantics recover the requested proxy value.  Post-mask perturbation sensitivity
and factual-versus-predicted-descriptor generation evaluations quantify part of
this risk, but no Wasserstein or clinical-semantic certificate is claimed.

\section{Cross-Fitted Factual-Only AIPW Policy Evaluation}
\label{sec:factual-aipw}

\subsection{Censoring-adjusted outcome}

Under Assumption~\ref{ass:censoring}, define
\begin{equation}
 \widetilde U_\tau
 =\frac1\tau\int_0^\tau
 \frac{\1\{\widetilde T\ge t\}}{G(t)}\,dt.
 \label{eq:ipcw-rmst}
\end{equation}
Then
$\E(\widetilde U_\tau\mid Z)=\E(U_\tau\mid Z)$; the proof is given in
Appendix~\ref{app:censoring}.  If no observations are censored before $\tau$,
$G\equiv1$ and Eq.~\eqref{eq:ipcw-rmst} reduces to observed normalized RMST.
The implementation substitutes the empirical marginal Kaplan--Meier estimate
$\widehat G$ and stores the day-scale pseudo-outcome
$\widetilde R_\tau=\tau\widetilde U_\tau$.

\subsection{Factual clinical outcome nuisance and AIPW estimator}

For a deterministic policy $\pi$ frozen before evaluation, define the
factual-policy target
\begin{equation}
 V^{\mathrm{fact}}_\tau(\pi)=\E_H\{\mu^{\mathrm{fact}}_{\pi(H)}(H)\},
 \qquad
 \mu^{\mathrm{fact}}_a(h)=\E(\widetilde R_\tau\mid H=h,A=a).
 \label{eq:factual-target}
\end{equation}
Here $\widetilde R_\tau=\tau\widetilde U_\tau$ is the IPCW RMST pseudo-outcome
in weeks.  In contrast to the planning-time $Q_\tau$, the factual nuisance is
fit from pre-treatment clinical history, the factual regimen, and observed
survival only.  The released implementation uses the prespecified clinical
subset $C\subset H$ and writes this restricted model as
$\widehat\mu^{\mathrm{fact}}_a(C)$; it never receives a generated world,
response descriptor, pseudo-mask, post-treatment CT, or other post-treatment
variable.  The outcome-regression correctness arm below therefore requires
that $C$ be sufficient for $\mu^{\mathrm{fact}}_a(H)$; the propensity-correct
arm does not rely on that restriction.

Let $\widehat\mu^{\mathrm{fact},(-k)}$ and $\widehat e^{(-k)}$ be nuisance
estimates trained without fold $k$, and let $k(i)$ be the fold containing
patient $i$.  For a policy learned without patient $i$, define
\begin{align}
 \widehat\psi_i(\pi)
 &=\widehat\mu^{\mathrm{fact},(-k(i))}_{\pi(H_i)}(C_i) \nonumber\\
 &\quad+
 \frac{\1\{A_i=\pi(H_i)\}}
 {\widehat e_{A_i}^{(-k(i))}(H_i)}
  \left\{\widetilde R_{\tau,i}
  -\widehat\mu^{\mathrm{fact},(-k(i))}_{A_i}(C_i)\right\},
 \label{eq:aipw-score}\\
 \widehat V_{\mathrm{AIPW}}(\pi)
 &=\frac1n\sum_{i=1}^n\widehat\psi_i(\pi).
 \label{eq:aipw-estimator}
\end{align}

\begin{theorem}[Cross-fitted factual-only AIPW]
\label{thm:factual-aipw}
Assume \ref{ass:consistency}--\ref{ass:censoring}, i.i.d. patients, the true
censoring survival $G$, a policy $\pi$ fixed conditional on the training folds,
and uniformly bounded inverse estimated propensities.  Let $e$ and $\mu$ denote
the true treatment and factual outcome nuisances, respectively.

\begin{enumerate}
 \item The population expectation of the score in Eq.~\eqref{eq:aipw-score}
 equals $V^{\mathrm{fact}}_\tau(\pi)$ if either $\widehat e=e$ almost surely
 or $\widehat\mu^{\mathrm{fact}}=\mu^{\mathrm{fact}}$ almost surely.  For
 the latter condition, the implemented clinical subset $C$ must be sufficient
 for the factual conditional outcome regression.
 \item If the cross-fitted nuisance estimates satisfy
 \begin{equation}
  \norm{\widehat e-e}_2\,
  \norm{\widehat\mu-\mu}_2=o_p(n^{-1/2}),
  \qquad
  \norm{\widehat e-e}_2+\norm{\widehat\mu-\mu}_2=o_p(1),
  \label{eq:product-rate}
 \end{equation}
 then
 \begin{equation}
  \sqrt n\{\widehat V_{\mathrm{AIPW}}(\pi)-V^{\mathrm{fact}}_\tau(\pi)\}
  \rightsquigarrow \mathcal N(0,\sigma_\pi^2),
  \label{eq:aipw-clt}
 \end{equation}
 where
 \begin{equation}
  \sigma_\pi^2
  =\operatorname{Var}\left[
    \mu^{\mathrm{fact}}_{\pi(H)}(H)
    +\frac{\1\{A=\pi(H)\}}{e_A(H)}
     \{\widetilde R_\tau-\mu^{\mathrm{fact}}_A(H)\}
  \right].
 \end{equation}
\end{enumerate}
\end{theorem}

Theorem~\ref{thm:factual-aipw} is an application of augmented
inverse-propensity estimation and cross-fitting
\cite{bang2005dr,chernozhukov2018dml}.  DeepACE similarly uses targeting to
obtain doubly robust causal estimation for patient trajectories
\cite{frauen2023deepace}.  CaMeWM's paper-specific contribution is different:
the world model constructs the frozen, support-restricted policy, while the
factual clinical nuisance makes the real-data value estimator auditable without
reusing generated states.

For the ideal score with known $G$, an asymptotic standard error is obtained
from the cross-fitted scores:
\begin{equation}
 \widehat\sigma_\pi^2
 =\frac1n\sum_{i=1}^n
 \{\widehat\psi_i(\pi)-\widehat V_{\mathrm{AIPW}}(\pi)\}^2,
 \qquad
 \widehat V_{\mathrm{AIPW}}(\pi)
 \pm z_{1-\alpha/2}\frac{\widehat\sigma_\pi}{\sqrt n}.
 \label{eq:aipw-ci}
\end{equation}

The real-data implementation differs from this ideal estimator in two explicit
ways.  First, it truncates the observed-action denominator at
$\epsilon_e=0.01$:
\begin{equation}
 \widehat e^{\mathrm{tr}}_a(h)=\max\{\widehat e_a(h),0.01\}.
 \label{eq:aipw-truncation}
\end{equation}
Second, it estimates $G$ by a marginal Kaplan--Meier curve.  Outer-fold
cross-validation makes the propensity, factual clinical outcome predictions,
and patient-specific policy out of fold, but the pooled censoring curve is not an
additional learned cross-fitted network.  Exact double robustness applies to
the untruncated score when $G$ is correct; truncation can introduce bias when it
is active.  This is why the semi-synthetic $2\times2$ identity check disables
clipping, while the observational analysis reports support, effective sample
size, and clipping sensitivity.  Real-data uncertainty is reported with a
patient-level nonparametric bootstrap of the out-of-fold factual AIPW value gain
rather than Eq.~\eqref{eq:aipw-ci}; an analytic interval would need to account for
estimation of $G$.

\section{Implemented Support-Constrained Monte Carlo Planning}

The released planner first forms the estimated support set
\begin{equation}
 \widehat\Asup(h)=\{a\in\Aset:\widehat e_a(h)\ge0.05\}.
 \label{eq:estimated-support}
\end{equation}
For each retained regimen it draws $K=8$ diffusion worlds, computes day-scale
RMST values $\widehat R_{\tau k}(a,h)$, and uses
\begin{equation}
 \overline R_{\tau K}(a,h)
 -1.96\frac{s_K(a,h)}{\sqrt K},
 \qquad
 s_K^2(a,h)=\frac1{K-1}\sum_{k=1}^K
 \{\widehat R_{\tau k}(a,h)-\overline R_{\tau K}(a,h)\}^2,
 \label{eq:implemented-mc-lcb}
\end{equation}
as a conservative Monte Carlo ranking score.  This score captures variation
across diffusion draws conditional on one fitted descriptor and survival model.
It does not include fold-to-fold parameter uncertainty, descriptor uncertainty,
survival-model uncertainty, or multiplicity across actions.

The following theorem records the stronger calibration target.  Let
$b(a,h)\ge0$ be a calibrated simultaneous radius and define
\begin{equation}
 \pi_{\mathrm{LCB}}(h)
 \in\arg\max_{a\in\Asup(h)}
 \{\widehat V_\tau(a,h)-b(a,h)\}.
 \label{eq:lcb-policy}
\end{equation}

\begin{theorem}[Supported LCB near-optimality]
\label{thm:lcb}
Suppose that, with probability at least $1-\delta$,
\begin{equation}
 |\widehat V_\tau(a,h)-V_\tau(a,h)|\le b(a,h)
 \quad\text{simultaneously for all }a\in\Asup(h).
 \label{eq:uniform-event}
\end{equation}
On this event,
\begin{equation}
 V_\tau(\pi_{\mathrm{LCB}}(h),h)
 \ge
 \max_{a\in\Asup(h)}V_\tau(a,h)
 -2\max_{a\in\Asup(h)}b(a,h).
 \label{eq:lcb-bound}
\end{equation}
\end{theorem}

The guarantee is conditional on calibrated coverage in
Eq.~\eqref{eq:uniform-event}, correct support classification, and use of the
same supported action set on both sides of Eq.~\eqref{eq:lcb-bound}.  The normal
Monte Carlo score in Eq.~\eqref{eq:implemented-mc-lcb} does not establish these
conditions.  Accordingly, Theorem~\ref{thm:lcb} is a target guarantee for a
future calibrated ensemble and is not claimed for the current real-data
planner.  Current evidence is limited to support-threshold sensitivity,
world-sampling convergence, and observational off-policy evaluation.

\section{What Cannot Be Identified}

\begin{proposition}[Unmeasured-confounding non-identifiability]
\label{prop:nonidentification}
Without Assumption~\ref{ass:exchangeability}, the observational distribution
does not generally identify $V_\tau(a)$, even when both treatment levels have
positive marginal probability.
\end{proposition}

\begin{remark}[No unique individual cross-world CT]
\label{rem:crossworld}
Theorem~\ref{thm:identification} identifies the marginal interventional
distribution $p(X^+(a)\mid H=h)$ under its assumptions.  It does not identify
the joint coupling
$p\{X^+(a),X^+(a')\mid H=h\}$ from one factual trajectory per patient.
Different copulas can share the same identified marginals while assigning
different alternate images to the same individual.  A claim of a unique
patient-specific counterfactual CT would require additional rank-preservation,
invertibility, or shared-exogenous-noise assumptions and corresponding
validation.  CaMeWM therefore reports samples from an interventional
distribution, not the true unobserved scan.
\end{remark}

\section{Data, Implementation, and Artifact Correspondence}
\label{sec:artifact-correspondence}

This section records the executable contract behind the mathematical objects.
It is deliberately more specific than the main paper: an implementation change
that violates this contract requires a corresponding change to the claim or
theory, even if predictive metrics improve.

\subsection{Outcome-blind image-state construction}

Image-state construction is completed before survival outcomes are accessed.
In particular, cohort exclusions, scan pairing, label-semantic checks,
registration, cropping, and pseudo-follow-up-mask QC do not use overall-survival
time, event status, or response-derived outcome labels.  The factual regimen is
retained as the exposure, but it never determines image geometry, crop position,
or mask acceptance.

HCC-TACE-Seg provides a multi-label baseline segmentation.  We use source
label~2 (\emph{Mass/HCC tumor}) as the binary lesion state, rather than the
union of liver, vessel, aorta, and tumor labels.  Source-label counts,
geometric checks, volume strata, and CT overlays audit this choice.  Of the
80 episodes with one of the two predeclared supported regimens, one has an
empty baseline label-2 mask and is excluded, leaving 79 episodes before
cross-validation.

Paired scans are selected using acquisition-date and phase rules, checked in
physical coordinates, and rigidly registered.  CT volumes are resampled to
$1\,\mathrm{mm}$ isotropic spacing, clipped to $[-175,600]$ HU, and normalized;
masks use nearest-neighbor resampling.  Both time points use the same
$96^3$ physical crop centered on the baseline lesion, preventing independent
recentring of pre- and post-treatment patches.  For each outer fold, the
follow-up segmenter is selected only on the outer-training validation split;
held-out patients are processed once with the selected model and cannot affect
model selection or QC thresholds.

\subsection{Fold-local fitting and model selection}

Evaluation uses five patient-level outer folds with a prespecified patient
partition.
Within each outer-training partition, 20\% is reserved for validation; outer-test
patients do not select checkpoints, preprocessing statistics, or thresholds.
All neural modules use AdamW for at most 200 epochs with stage-specific early
stopping.  Folds are scheduled across four NVIDIA RTX A6000 devices, but each
fold remains a statistically independent training and evaluation unit.

\begin{table}[htbp]
\centering
\small
\caption{Validation criteria used to select fold-local checkpoints.  Selection
metrics are fixed before examining outer-test results.}
\label{tab:selection-contract}
\begin{tabular}{p{0.21\linewidth}p{0.36\linewidth}p{0.34\linewidth}}
\toprule
Module & Checkpoint criterion & Safeguard \\
\midrule
Follow-up segmenter & Maximum validation Dice & Minimum 40 epochs; patience 20 \\
Response predictor & Minimum standardized validation MAE & Minimum 40 epochs; patience 25 \\
Joint propensity & Minimum validation cross-entropy & Minimum 30 epochs; patience 20 \\
Residual DDPM & Minimum sampled tumor-region MAE, checked every 10 epochs &
Minimum 60 epochs; patience 5 checks \\
Response-Cox survival & Maximum validation C-index; validation loss only when C-index is undefined &
Minimum 50 epochs; patience 30 \\
Planner & No test-set tuning; $K=8$, support threshold 0.05 &
Same locked values for every fold and method \\
\bottomrule
\end{tabular}
\end{table}

\subsection{Full experimental protocol and statistical reporting}
\label{sec:full-experimental-protocol}

The public HCC-TACE-Seg release contains 105 patients with pre/post-TACE
multiphasic CT, treatment, response, and survival records.  The manifest
builder retains \CuratedN{} episodes with the required local entries.  We
predeclare cisplatin--doxorubicin--mitomycin-C and doxorubicin LC beads as the
two represented joint regimens.  Action screening retains \ActionSupportedN{}
episodes (50/30); one empty released baseline tumor mask is then excluded,
leaving \SupportedN{} episodes (50/29).  No exclusion uses response or survival
outcomes.  Follow-up selection matches the baseline acquisition identifier
before physical-coordinate registration.  Four rotation test-time views and an
intensity fallback for predicted masks below 16 voxels are recorded in the QC
artifacts; response descriptors are extracted only after registration and
follow-up mask prediction.

All imaging modules are optimized with AdamW for at most 200 epochs and
stage-specific early stopping.  The response and diffusion encoders use 32 base
channels, and the segmenter and propensity image encoder use 24 and 16,
respectively.  Diffusion and response training use automatic mixed precision;
the response-Cox head is regularized and low-dimensional.  Experiments use
Linux, Python 3.9, PyTorch $\geq 2.2$, and four NVIDIA RTX A6000 GPUs (48 GB).
Outer folds may occupy different devices but never share test data.  Exact
package constraints, batch sizes, and launch commands are included with the
released materials.

Generation comparators are persistence ($I^+=I^-$), deterministic 3D U-Net,
conditional VAE-3D, mask-conditioned DDPM, and MeWM action-conditioned
DDPM without the response descriptor $M$ or causal adjustment.  We report
global and tumor HU-MAE, PSNR, 3D SSIM, and change-map $\Delta$SSIM.  Survival
comparators are CoxPH, DeepSurv, IPTW-Cox, MeWM, and CaMeWM; reported
metrics are Uno's C-index, integrated Brier score (IBS), RMST MAE, and
Harrell's C-index.  Planning uses the same comparator set.  Historical-regimen
agreement, macro-F1, and balanced accuracy are behavioural references only,
not counterfactual ground truth.

All uncertainty summaries use patient-level OOF values, never individual CT
slices or diffusion samples.  IBS integrates IPCW Brier scores at 12, 24, and
36 months.  Patient-bootstrap intervals use 2,000 resamples; paired method
comparisons resample the same patient identifiers.  The factual-only AIPW
policy estimator uses support threshold $\eta=0.05$ and correction-weight cap
10.  Its clinical outcome nuisance excludes generated states.  Overlap weighting
is reported solely as a variance diagnostic and is never selected to improve the
reported policy value.  Supplementary Fig.~S1 reports propensity calibration,
overlap, and balance; Fig.~S2 reports generated-to-factual survival transport;
Fig.~S3 reports stability for $K\in\{2,4,8,16,32\}$; and Fig.~S4 audits the
non-certifying association between pseudo-label confidence and OOF
tumor-region error.

\subsection{Reproducibility materials and result traceability}

The anonymized accompanying materials are organized around analytical stages,
not repository-specific file names.  They provide the fixed split definition,
configuration values, environment instructions, public-data access guidance,
and the fold-level prediction files required to recreate every reported OOF
summary.  Table~\ref{tab:artifact-map} states the audit trail at the level
needed to assess reproducibility; a checkpoint alone is not evaluation
evidence.

\begin{table}[htbp]
\centering
\small
\caption{Reproducibility materials and audit trail.  The table intentionally
describes analytical stages rather than release-specific script or archive names.}
\label{tab:artifact-map}
\begin{tabular}{p{0.25\linewidth}p{0.35\linewidth}p{0.31\linewidth}}
\toprule
Analytical stage & Reproducible input or operation & Required evidence \\
\midrule
Image-state construction &
label semantics, scan pairing, registration, common physical crop &
cohort flow, exclusions, label and registration QC \\
Fold-local model fitting &
training/validation partitions, fixed selection criteria, frozen outer test fold &
selection records and fold-level predictions \\
Response and transition estimation &
response descriptor, causal anchor, residual transition, world sampling &
module evaluations, sampled worlds, and qualitative cases \\
Survival and planning &
factual-state survival fit, supported action scoring, factual-only AIPW &
survival, planning, support, and policy-value summaries \\
OOF aggregation &
patient-level pooling and paired bootstrap resampling &
reported tables, figures, and confidence intervals \\
\bottomrule
\end{tabular}
\end{table}

The accompanying materials specify the software environment, fixed five-fold
OOF partition, early-stopping criteria, and the commands that produce the
fold-level summaries.  They do not redistribute raw CT, clinical records,
survival outcomes, or trained checkpoints.  Instead, they provide data-access
instructions, cohort-construction logic, and validation checks so that an
authorized local copy of the public collection can be reconstructed and
audited.  Reported manuscript values are derived only from frozen OOF
predictions and patient-level bootstrap summaries.

\section{Theory-to-Experiment Contract}
\label{sec:contract}

Claims are enabled only when both the corresponding implementation and evidence
are present.
\begin{description}
 \item[Theorem~\ref{thm:identification}.]
 Implement joint-regimen propensity, a pre-treatment adjustment set, and RMST.
 Report a pre-specified DAG, the confounder table, and the explicit
 unmeasured-confounding limitation.
 \item[Theorem~\ref{thm:value-error}.]
 Evaluate response-descriptor, transition, and survival factors separately.  Report module
 ablations, per-factor errors, semi-synthetic value error, and empirical bound
 coverage.
 \item[Lemma~\ref{lem:mc}.]
 Draw independent world-model samples for each regimen.  The semi-synthetic
sampling experiment verifies $K^{-1/2}$ scaling; the real planner uses $K=8$
matched diffusion draws and one deterministic response-descriptor prediction.
 \item[Theorem~\ref{thm:factual-aipw}.]
 Use outer-fold out-of-fold policies, joint propensity, marginal-KM IPCW, and
 the factual-only AIPW score.  The exact unclipped identity is validated in all
 four correct/wrong nuisance combinations semi-synthetically.  Observational
 results must additionally disclose denominator truncation and bootstrap
 uncertainty.
 \item[Theorem~\ref{thm:lcb}.]
 The current implementation satisfies support restriction but not simultaneous
 uncertainty calibration.  Therefore report Eq.~\eqref{eq:implemented-mc-lcb}
 as a Monte Carlo ranking score, not as evidence that the theorem's
 near-optimality guarantee holds.
\end{description}

For observational HCC-TACE-Seg experiments, the minimum causal diagnostics are:
joint-regimen frequencies; propensity overlap; effective sample size
$(\sum_iw_i)^2/\sum_iw_i^2$; clipping rate; pre/post-weighting standardized
mean differences; supported-action fraction; and cross-fitted factual-only
AIPW bootstrap intervals.  Treatment agreement F1 remains a secondary imitation
metric.  Out-of-fold factual-only AIPW value is the primary implemented planning evaluation;
clinician assessment would be additional evidence rather than an output of the
released code.

\section{Semi-Synthetic Validation Required by the Theory}

Real observational data do not reveal ground-truth counterfactual images or
effects.  The supplement therefore uses a semi-synthetic benchmark that
resamples the observed HCC clinical-covariate distribution while generating
treatments, response descriptors, transitions, and outcomes from known SCMs.
It does not train or evaluate the full image-level CaMeWM checkpoints.

\begin{enumerate}
 \item \textbf{Identifiability recovery.} Increase simulated sample size and
 compare estimated versus true ATE and fixed-policy value.
 \item \textbf{Factual-AIPW $2\times2$ identity check.} In the known SCM,
 define a correct propensity or outcome nuisance as its oracle population
 function and a misspecified nuisance as an intercept-only function.  Evaluate
 all four cells without weight clipping.  The DR estimator should be centered
 at zero bias in both one-correct cells and should fail in the both-wrong cell.
 Estimated-nuisance convergence and clipping are assessed separately so that
 they do not obscure the population identity in Theorem~\ref{thm:factual-aipw}.
 \item \textbf{Overlap stress test.} Reduce $\min e_a(H)$ and report effective
 sample size and the fraction with propensities in $[0.05,0.95]$.
 \item \textbf{Factor-error test.} Use an exactly enumerable bounded binary SCM,
 independently contaminate response-descriptor, transition, and outcome factors, and
 compute their realized total-variation or sup-norm errors.  Verify condition by
 condition that the maximum causal-value error is bounded by
 $\varepsilon_M+\varepsilon_X+\varepsilon_Q$ in
 Eq.~\eqref{eq:value-bound}, and report empirical bound coverage.
 \item \textbf{Sampling test.} Vary $K$ and verify Monte Carlo deviation and
 policy instability decrease at the predicted $K^{-1/2}$ rate.
 \item \textbf{Hidden-confounding sensitivity.} Add an unobserved common cause of
 treatment and outcome to demonstrate the limitation in
 Proposition~\ref{prop:nonidentification}.
\end{enumerate}

\section{Proof of Theorem~\ref{thm:identification}}
\label{app:identification}

\begin{proof}
Fix a supported $(a,h)$.  By conditional exchangeability,
\begin{equation}
 \mathcal L\{M(a),X^+(a),U_\tau(a)\mid H=h\}
 =\mathcal L\{M(a),X^+(a),U_\tau(a)\mid H=h,A=a\}.
 \label{eq:proof-exchange}
\end{equation}
By consistency, the right-hand side equals the observed conditional law
$\mathcal L(M,X^+,U_\tau\mid H=h,A=a)$.  Therefore,
\begin{align}
 V_\tau(a,h)
 &=\E\{U_\tau(a)\mid H=h\}\\
 &=\E\{U_\tau\mid H=h,A=a\}\\
 &=\int\!\!\int
   \E(U_\tau\mid h,a,m,x)
   p(x,m\mid h,a)\,dx,dm\\
 &=\int\!\!\int
   Q_\tau(h,a,m,x)
   p_X(x\mid h,a,m)p_M(m\mid h,a)\,dx,dm.
\end{align}
The last equality is the chain-rule response factorization in
Eq.~\eqref{eq:causal-factorization}; it does not assert that the image-derived
descriptor is measured before the follow-up CT.  Positivity ensures that all
observed conditional laws on the right are defined for actions considered by
the policy.  Finally, by
definition, maximizing the identified value over $\Asup(h)$ produces the
optimal supported action.
\end{proof}

\section{Proof of Theorem~\ref{thm:value-error}}
\label{app:value-error}

\begin{proof}
Suppress $(h,a)$ in the notation and write
$P=P_MP_{X\mid M}$ and
$\widehat P=\widehat P_M\widehat P_{X\mid M}$.  Add and subtract
$\int Q\,d\widehat P$:
\begin{align}
 |\widehat V-V|
 &\le
 \left|\int(\widehat Q-Q)\,d\widehat P\right|
 +\left|\int Q\,d\widehat P-\int Q\,dP\right|\\
 &\le \varepsilon_Q
 +\left|\int q(m)\,d\widehat P_M(m)
          -\int \widetilde q(m)\,dP_M(m)\right|,
 \label{eq:bound-step1}
\end{align}
where
$q(m)=\int Q(m,x)d\widehat P_{X\mid m}(x)$ and
$\widetilde q(m)=\int Q(m,x)dP_{X\mid m}(x)$.
Add and subtract $\int\widetilde q\,d\widehat P_M$ in the second term.  Since
$0\le Q,q,\widetilde q\le1$,
\begin{align}
 \left|\int(q-\widetilde q)d\widehat P_M\right|
 &\le\sup_m
 \left|\int Q(m,x)d\widehat P_{X\mid m}
       -\int Q(m,x)dP_{X\mid m}\right|
 \le\varepsilon_X,\\
 \left|\int\widetilde q\,d\widehat P_M
       -\int\widetilde q\,dP_M\right|
 &\le\varepsilon_M.
\end{align}
The last inequalities use the variational characterization of total variation
for $[0,1]$-valued functions.  Combining them with
Eq.~\eqref{eq:bound-step1} proves Eq.~\eqref{eq:value-bound}.
\end{proof}

\section{Proof of Corollary~\ref{cor:plugin-regret}}

\begin{proof}
For fixed $h$, let $a^*=\pi^*(h)$ and
$\widehat a=\widehat\pi(h)$.  Then
\begin{align}
 V(a^*,h)-V(\widehat a,h)
 &=V(a^*,h)-\widehat V(a^*,h)
  +\widehat V(a^*,h)-\widehat V(\widehat a,h)
  +\widehat V(\widehat a,h)-V(\widehat a,h)\\
 &\le \varepsilon(h)+0+\varepsilon(h).
\end{align}
The middle term is nonpositive because $\widehat a$ maximizes $\widehat V$ on
the same supported action set.
\end{proof}

\section{Proofs of Lemma~\ref{lem:mc} and Corollary~\ref{cor:sampled-regret}}

\begin{proof}[Proof of Lemma~\ref{lem:mc}]
Conditionally on the fitted world model, the summands in
Eq.~\eqref{eq:mc-value} are independent, lie in $[0,1]$, and have mean
$\widehat V_\tau(a,h)$.  Hoeffding's inequality gives
Eq.~\eqref{eq:mc-hoeffding}.  Applying the union bound to $J$ supported actions
gives probability at most $2J\exp(-2Kr^2)$ that any action violates the bound.
Setting this expression to $\delta$ yields Eq.~\eqref{eq:mc-radius}.
\end{proof}

\begin{proof}[Proof of Corollary~\ref{cor:sampled-regret}]
On the simultaneous Monte Carlo event, for every supported action,
\begin{equation*}
 |\widehat V_{\tau,K}(a,h)-V_\tau(a,h)|
 \le
 |\widehat V_{\tau,K}-\widehat V_\tau|
 +|\widehat V_\tau-V_\tau|
 \le r_K+\varepsilon(h).
\end{equation*}
Repeat the three-term argmax decomposition in the proof of
Corollary~\ref{cor:plugin-regret} with uniform error
$r_K+\varepsilon(h)$.
\end{proof}

\section{Proof of Theorem~\ref{thm:factual-aipw}}
\label{app:factual-aipw}

\begin{proof}
We first prove double robustness at the population level.  Conditional on
$H=h$, write $a_\pi=\pi(h)$ and let barred symbols denote arbitrary candidate
nuisances.  Since
$\E(\widetilde U_\tau\mid H=h,A=a)=\mu_a(h)$,
\begin{align}
 &\E\left[
 \bar\mu_{a_\pi}(h)
 +\frac{\1\{A=a_\pi\}}{\bar e_A(h)}
  \{\widetilde U_\tau-\bar\mu_A(h)\}
 \;\middle|\;H=h\right]\\
 &=\bar\mu_{a_\pi}(h)
 +\frac{e_{a_\pi}(h)}{\bar e_{a_\pi}(h)}
  \{\mu_{a_\pi}(h)-\bar\mu_{a_\pi}(h)\}\\
 &=\mu_{a_\pi}(h)
 +\frac{\bar e_{a_\pi}(h)-e_{a_\pi}(h)}
       {\bar e_{a_\pi}(h)}
  \{\bar\mu_{a_\pi}(h)-\mu_{a_\pi}(h)\}.
 \label{eq:aipw-remainder}
\end{align}
The remainder is zero if either $\bar e=e$ or $\bar\mu=\mu$.  Integrating over
$H$ proves part (i).

For part (ii), let
\begin{equation*}
 \psi(O;\eta,\pi)
 =\mu_{\pi(H)}(H)
 +\frac{\1\{A=\pi(H)\}}{e_A(H)}
  \{\widetilde U_\tau-\mu_A(H)\},
\end{equation*}
and let $\eta_0=(e,\mu)$.  Cross-fitting makes each validation-fold score
conditionally independent of the nuisance training sample.  Decompose
\begin{align}
 \widehat V_{\mathrm{AIPW}}-V
 &=(\mathbb P_n-P)\psi(O;\eta_0,\pi)\\
 &\quad +(\mathbb P_n-P)
  \{\psi(O;\widehat\eta,\pi)-\psi(O;\eta_0,\pi)\}\\
 &\quad +P\{\psi(O;\widehat\eta,\pi)-\psi(O;\eta_0,\pi)\}.
 \label{eq:dml-decomp}
\end{align}
The first term obeys the classical central limit theorem.  Conditional
cross-fitting plus $L_2$ consistency makes the second term $o_p(n^{-1/2})$.
By Eq.~\eqref{eq:aipw-remainder}, Cauchy--Schwarz, and bounded inverse propensity,
the absolute third term is bounded by
\begin{equation*}
 C\norm{\widehat e-e}_2\norm{\widehat\mu-\mu}_2
 =o_p(n^{-1/2}).
\end{equation*}
Multiplying Eq.~\eqref{eq:dml-decomp} by $\sqrt n$ and applying Slutsky's
theorem gives Eq.~\eqref{eq:aipw-clt}.  Consistency of the empirical score
variance follows from the same cross-fitting and moment conditions.
\end{proof}

\section{Proof of Theorem~\ref{thm:lcb}}

\begin{proof}
Fix $h$ and let
$a^*\in\arg\max_{a\in\Asup(h)}V_\tau(a,h)$ and
$\widehat a=\pi_{\mathrm{LCB}}(h)$.  On event
\eqref{eq:uniform-event},
\begin{align}
 V_\tau(\widehat a,h)
 &\ge \widehat V_\tau(\widehat a,h)-b(\widehat a,h)\\
 &\ge \widehat V_\tau(a^*,h)-b(a^*,h)\\
 &\ge V_\tau(a^*,h)-2b(a^*,h)\\
 &\ge V_\tau(a^*,h)-2\max_{a\in\Asup(h)}b(a,h).
\end{align}
The second inequality is the definition of the LCB maximizer.
\end{proof}

\section{Proof of Proposition~\ref{prop:nonidentification}}

\begin{proof}
Consider binary unobserved $U\sim\operatorname{Bernoulli}(1/2)$ and no measured
history.  In SCM~1, set $A=U$ and $Y=A$.  In SCM~2, set $A=U$ and $Y=U$.
Both SCMs induce the identical observational law
$\Pp(A=0,Y=0)=\Pp(A=1,Y=1)=1/2$, and both treatment levels have positive
marginal probability.  Under $\doop(A=1)$, however, SCM~1 gives $Y=1$ almost
surely, whereas SCM~2 gives $Y=U$ and hence $\E\{Y(1)\}=1/2$.  Thus the same
observational distribution is compatible with different causal values.  The
difference is unmeasured common causation of $A$ and $Y$ in SCM~2, which
violates conditional exchangeability.
\end{proof}

\section{Censoring Identity}
\label{app:censoring}

\begin{proof}
For fixed $t\le\tau$, conditional on $Z$ and using
$C^{\mathrm{cen}}\ind(T,Z)$,
\begin{align}
 \E\left[
 \frac{\1\{\widetilde T\ge t\}}{G(t)}
 \;\middle|\;Z\right]
 &=\frac{\Pp(T\ge t,C^{\mathrm{cen}}\ge t\mid Z)}{G(t)}\\
 &=\frac{\Pp(T\ge t\mid Z)G(t)}{G(t)}\\
 &=\Pp(T\ge t\mid Z).
\end{align}
Integrating from $0$ to $\tau$, dividing by $\tau$, and applying Fubini's
theorem yields
$\E(\widetilde U_\tau\mid Z)=\E(U_\tau\mid Z)$.
\end{proof}

\section{Reviewer-Facing Claim Boundaries}

The following wording is defensible only after the theory-to-experiment contract
in Section~\ref{sec:contract} is satisfied.

\begin{itemize}
 \item \emph{Defensible:} ``Under measured confounding, positivity, consistency,
 and censoring assumptions, CaMeWM identifies and estimates supported-regimen
 RMST values through a response-factorized g-formula.''
 \item \emph{Defensible:} ``The factor-wise error bound links response-proxy,
 transition, and survival-factor errors to policy regret; its empirical bound
 coverage is established in the known-ground-truth semi-synthetic experiment.''
 \item \emph{Not defensible:} ``The generated image is the patient's true
 counterfactual follow-up CT.''
 \item \emph{Not defensible:} ``The five image-derived response descriptors are
 validated biological mediators or identify natural direct and indirect
 effects.''
 \item \emph{Not defensible:} ``IPTW proves that treatment recommendations are
 causal.''  Weighting is valid only under the stated assumptions and adequate
 overlap.
 \item \emph{Not defensible:} ``Generated-world discrimination or transport
 diagnostics establish policy benefit.''  Those diagnostics assess the
 end-to-end transition-and-utility path; the real-data policy-value analysis is
 the separate factual-only AIPW evaluation in
 Section~\ref{sec:factual-aipw}.
 \item \emph{Not defensible:} ``The current $K=8$ normal Monte Carlo LCB has the
 calibrated finite-sample guarantee in Theorem~\ref{thm:lcb}.''  That theorem
 requires simultaneous coverage that the released planner does not estimate.
 \item \emph{Not defensible:} ``A theorem proves clinical benefit.''  Theorems
 establish statistical properties; clinical benefit requires external and
 prospective evidence.
\end{itemize}

\clearpage

\section{Empirical Protocol and Additional Results}
\label{sec:empirical-supplement}

\subsection{Cohort Construction and Locked Evaluation Protocol}

Table~\ref{tab:cohort-flow} records the outcome-blind cohort construction used
throughout the final experiment.  The released HCC-TACE-Seg collection contains
105 patients.  The manifest builder retained 99 episodes with the required
longitudinal local records; restricting the analysis to the two adequately
represented complete joint regimens retained 80 episodes.  One released
baseline mask was empty for source label~2 (Mass/HCC tumor) and was excluded
before model fitting, leaving 79 episodes.  No survival time, event indicator,
or response outcome was used in these exclusions.

\begin{table}[htbp]
\centering
\small
\caption{Outcome-blind cohort construction for the fixed five-fold OOF protocol.}
\label{tab:cohort-flow}
\begin{tabular}{lr}
\toprule
Stage & Episodes \\
\midrule
Public HCC-TACE-Seg release & 105 \\
Longitudinal local manifest with required imaging and clinical records & 99 \\
Two predeclared supported complete joint regimens & 80 \\
Empty released source label-2 tumor mask excluded & 1 \\
\midrule
Final patient-level OOF cohort & 79 \\
\bottomrule
\end{tabular}
\end{table}

The final protocol uses a prespecified patient-level five-fold outer
cross-validation partition.  Each outer-training partition reserves 20\% for
internal validation.  All transforms, imputation, standardization, segmentation,
propensity, response, image-transition, and survival modules are fitted
fold-locally.  Reported confidence intervals use 2,000 patient-level
bootstrap resamples of the pooled OOF predictions.  Table~\ref{tab:protocol}
lists the locked decisions that matter for reproducibility.

\begin{table}[htbp]
\centering
\small
\caption{Locked five-fold OOF evaluation decisions.}
\label{tab:protocol}
\begin{tabular}{p{0.32\linewidth}p{0.57\linewidth}}
\toprule
Item & Final specification \\
\midrule
Tumor semantics & Source label~2 (Mass/HCC tumor); empty source masks excluded \\
Spatial preprocessing & Rigid physical-grid registration; 1-mm isotropic spacing; shared $96^3$ baseline-tumor physical crop \\
Image scale & HU window $[-175,600]$ followed by the common normalization used for all methods \\
Follow-up lesion state & Fold-local segmenter, validation Dice selection, four-view test-time aggregation, recorded fallback QC \\
Treatment support & Two complete joint regimens; support threshold $\eta=0.05$ \\
World integration & $K=8$ diffusion worlds per candidate regimen \\
Survival utility & Validation-selected response-Cox model with fold-local Breslow calibration \\
Inference & Patient-level OOF predictions; paired 2,000-resample bootstrap intervals \\
\bottomrule
\end{tabular}
\end{table}

\subsection{Pre-treatment Covariates and Missing-Data Handling}
\label{sec:covariates-missingness}

Table~\ref{tab:covariates-missingness} specifies the eight pre-treatment
clinical covariates used in the joint-propensity and response-Cox inputs.
All are recorded before treatment assignment.  Preprocessing statistics are
fit using outer-training patients only: AFP is transformed as
$\log(1+\mathrm{AFP})$, missing values are replaced by the corresponding
outer-training median, and each covariate is median-centered and scaled by
its outer-training standard deviation.  A binary missingness indicator is
concatenated for every covariate.  Thus, no held-out patient contributes to
imputation or scaling.  The only incomplete selected covariate is baseline
tumor maximum diameter (41/79 episodes); its missingness indicator is retained
in both the propensity and utility inputs.

\begin{table}[htbp]
\centering
\small
\caption{Pre-treatment adjustment covariates and fold-local missing-data
handling in the final 79-episode cohort.  ``Missing'' is assessed before any
outer-fold preprocessing.  All covariates enter the joint propensity and the
response-Cox utility model together with their binary missingness indicators.}
\label{tab:covariates-missingness}
\setlength{\tabcolsep}{3pt}
\begin{tabular}{p{0.23\textwidth}p{0.23\textwidth}c p{0.37\textwidth}}
\toprule
Covariate & Release representation & Missing, $n$ (\%) & Fold-local processing \\
\midrule
Age & Continuous baseline age & 0 (0.0) & Training median-centering and SD scaling \\
Sex & Binary release code & 0 (0.0) & Training median-centering and SD scaling \\
AFP & Continuous reported baseline value & 0 (0.0) & $\log(1+\mathrm{AFP})$, then training median-centering and SD scaling \\
Child--Pugh class & Ordinal release code & 0 (0.0) & Training median-centering and SD scaling \\
BCLC stage & Ordinal release code & 0 (0.0) & Training median-centering and SD scaling \\
Tumor count & Binary release code & 0 (0.0) & Training median-centering and SD scaling \\
Tumor maximum diameter & Continuous baseline value & 41 (51.9) & Training-median imputation, then centering and SD scaling \\
Portal-vein invasion & Binary release code & 0 (0.0) & Training median-centering and SD scaling \\
\bottomrule
\end{tabular}
\end{table}

\subsection{Observed-Action Overlap and Weighting Diagnostics}

Figure~\ref{fig:s1-propensity} reports diagnostics for the fold-local joint
propensity models.  The empirical overlap plot shows that only 3\% of observed
regimens have estimated propensity below the prespecified 0.05 threshold.
Calibration and standardized-mean-difference plots are diagnostic rather than
proof of exchangeability: they assess measured clinical covariates only and
cannot rule out unmeasured confounding.  Missingness indicators are audited
separately so that the balance display remains interpretable.

\begin{figure}[htbp]
\centering
\includegraphics[width=\linewidth]{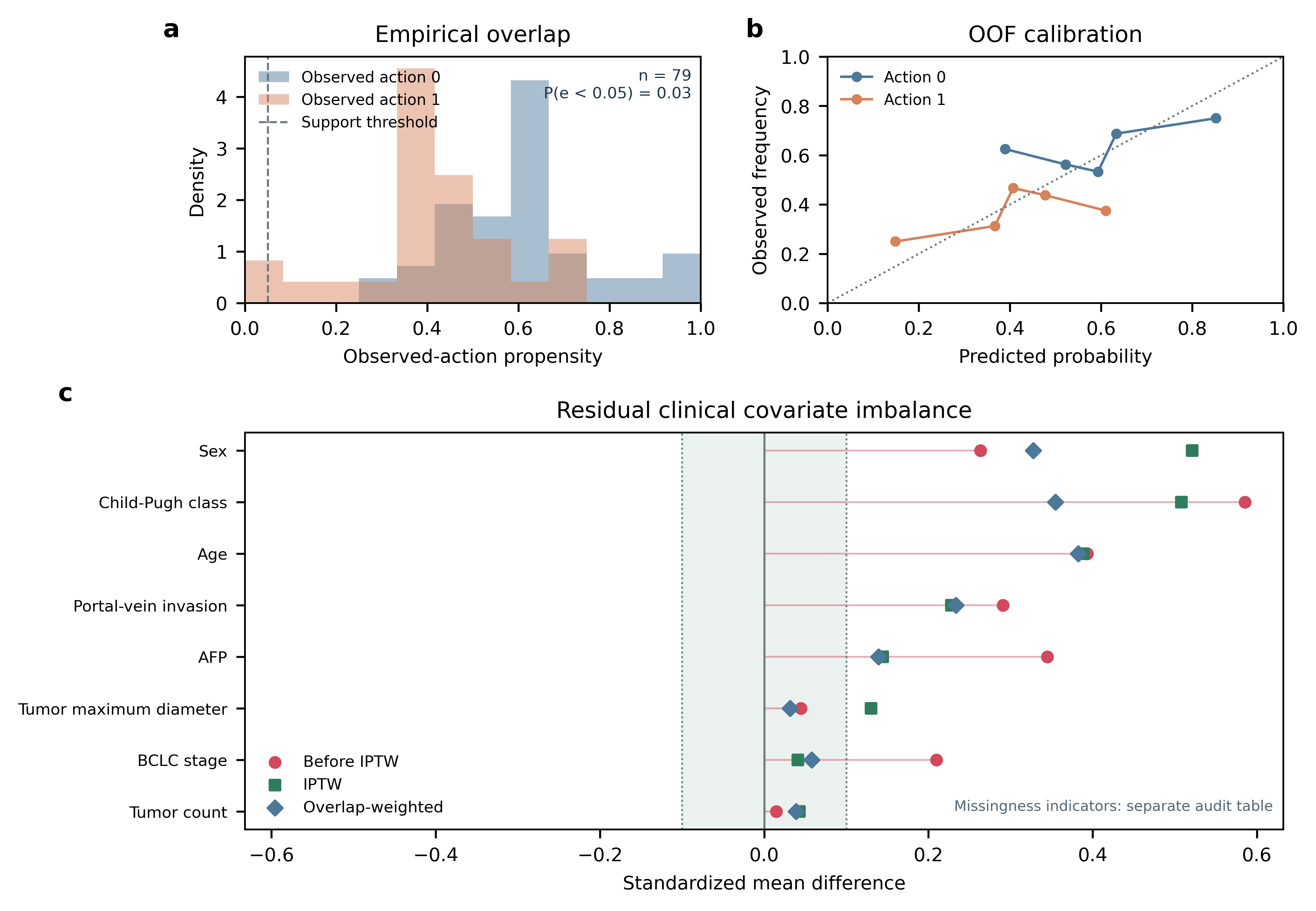}
\caption{\textbf{Observed-action propensity diagnostics.}
(a) Distribution of OOF observed-regimen propensities; the dashed line is the
predeclared support threshold. (b) Action-wise OOF calibration. (c) Standardized
mean differences before weighting, after IPTW, and after overlap weighting for
the non-missing clinical covariates.  These diagnostics characterize observed
support and measured-covariate balance; they do not establish exchangeability.}
\label{fig:s1-propensity}
\end{figure}

\subsection{Generated-State Transport and World-Count Stability}

The factual-state survival model and the planning-time generated-world path
serve different roles.  Figure~\ref{fig:s2-transport} therefore reports their
agreement rather than treating factual calibration as a certificate for
counterfactual calibration.  The paired generated-minus-factual C-index shift
is $+0.024$ with a bootstrap interval $[-0.024,0.068]$; this is a transport
diagnostic, not evidence of a clinical advantage.

\begin{figure}[htbp]
\centering
\includegraphics[width=\linewidth]{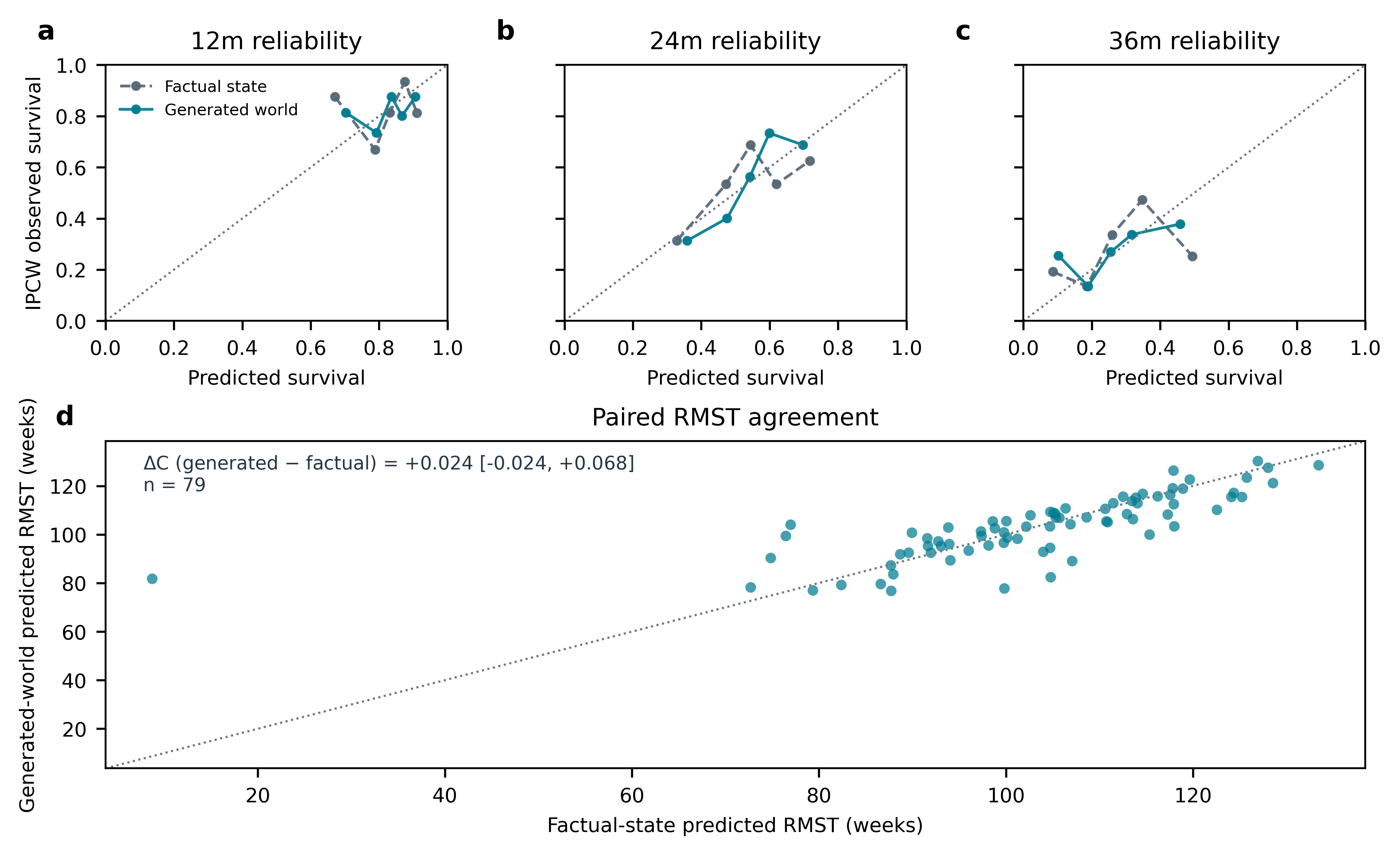}
\caption{\textbf{Factual-to-generated survival transport diagnostic.}
(a--c) IPCW reliability at 12, 24, and 36 months for factual and generated
states. (d) Patient-level agreement between factual-state and generated-world
predicted RMST.  The diagonal is identity.  The paired C-index change is
reported with a patient-bootstrap 95\% interval and is interpreted only as a
diagnostic for generated-state shift.}
\label{fig:s2-transport}
\end{figure}

Figure~\ref{fig:s3-world-count} checks the locked choice $K=8$ against larger
world counts.  Recommendations at $K=8$ agree with those at $K=32$ for 78 of
79 OOF patients (98.7\%).  The analysis establishes numerical recommendation
stability on this cohort; it does not calibrate the normal lower-confidence
score as a simultaneous uncertainty bound.

\begin{figure}[htbp]
\centering
\includegraphics[width=\linewidth]{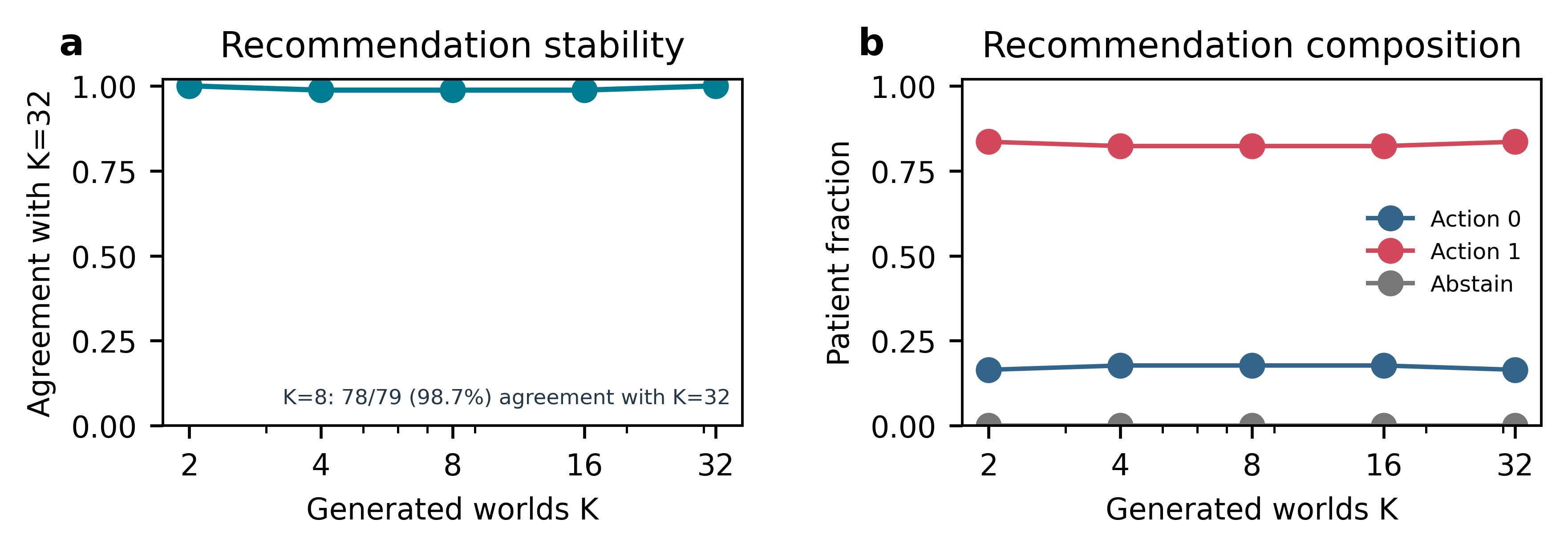}
\caption{\textbf{Stability with the number of sampled worlds.}
(a) OOF recommendation agreement with the $K=32$ reference over
$K\in\{2,4,8,16,32\}$.  At the locked $K=8$, 78/79 (98.7\%) recommendations
match $K=32$. (b) Corresponding recommended-regimen composition and abstention
fraction.  This supports the computational choice of $K=8$; it is not a
finite-sample confidence-calibration result.}
\label{fig:s3-world-count}
\end{figure}

\subsection{Pseudo-Label Audit}

Manual follow-up tumor masks are unavailable for all episodes.  The fold-local
segmenter therefore produces pseudo follow-up masks used to define the
operational response descriptors and tumor-region fidelity metric.
Figure~\ref{fig:s4-pseudomask} deliberately reports whether the available
confidence proxies rank OOF generation error.  The association is weak and its
bootstrap interval includes zero; accordingly, confidence is used as a
training weight and QC signal, not as a post hoc error certificate or a claim
of biological validity.

\begin{figure}[htbp]
\centering
\includegraphics[width=\linewidth]{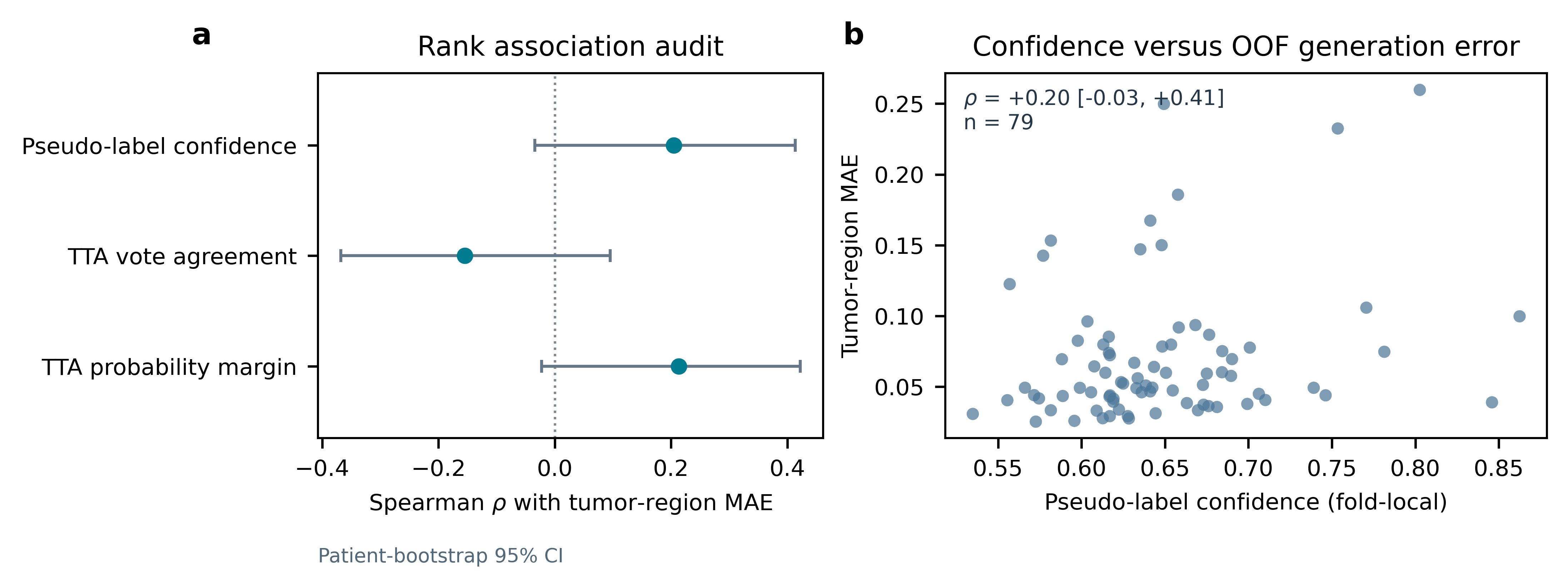}
\caption{\textbf{Pseudo follow-up mask confidence audit.}
(a) Spearman associations, with patient-bootstrap 95\% intervals, between
three fold-local confidence signals and OOF tumor-region MAE. (b) Direct
confidence-versus-error scatter.  The pseudo-label confidence association is
weak ($\rho=0.20$, 95\% CI $[-0.03,0.41]$), so these quantities are treated as
audits rather than calibrated error predictors.}
\label{fig:s4-pseudomask}
\end{figure}

\subsection{Known-Ground-Truth Semi-Synthetic Stress Tests}

The real cohort observes one regimen and one follow-up state per patient.  To
test estimator identities that cannot be checked in observational data, the
semi-synthetic experiment resamples clinical covariates and simulates
treatment, response, transition, and survival under a known structural model.
Figure~\ref{fig:s5-semisynthetic} evaluates recovery with increasing sample
size, the $2\times2$ double-robustness identity, positivity degradation,
factor-wise error propagation, Monte Carlo convergence, and hidden
confounding.  These panels validate the estimator under the specified
simulation mechanism; they do not constitute external clinical validation.

\begin{figure}[htbp]
\centering
\includegraphics[width=\linewidth]{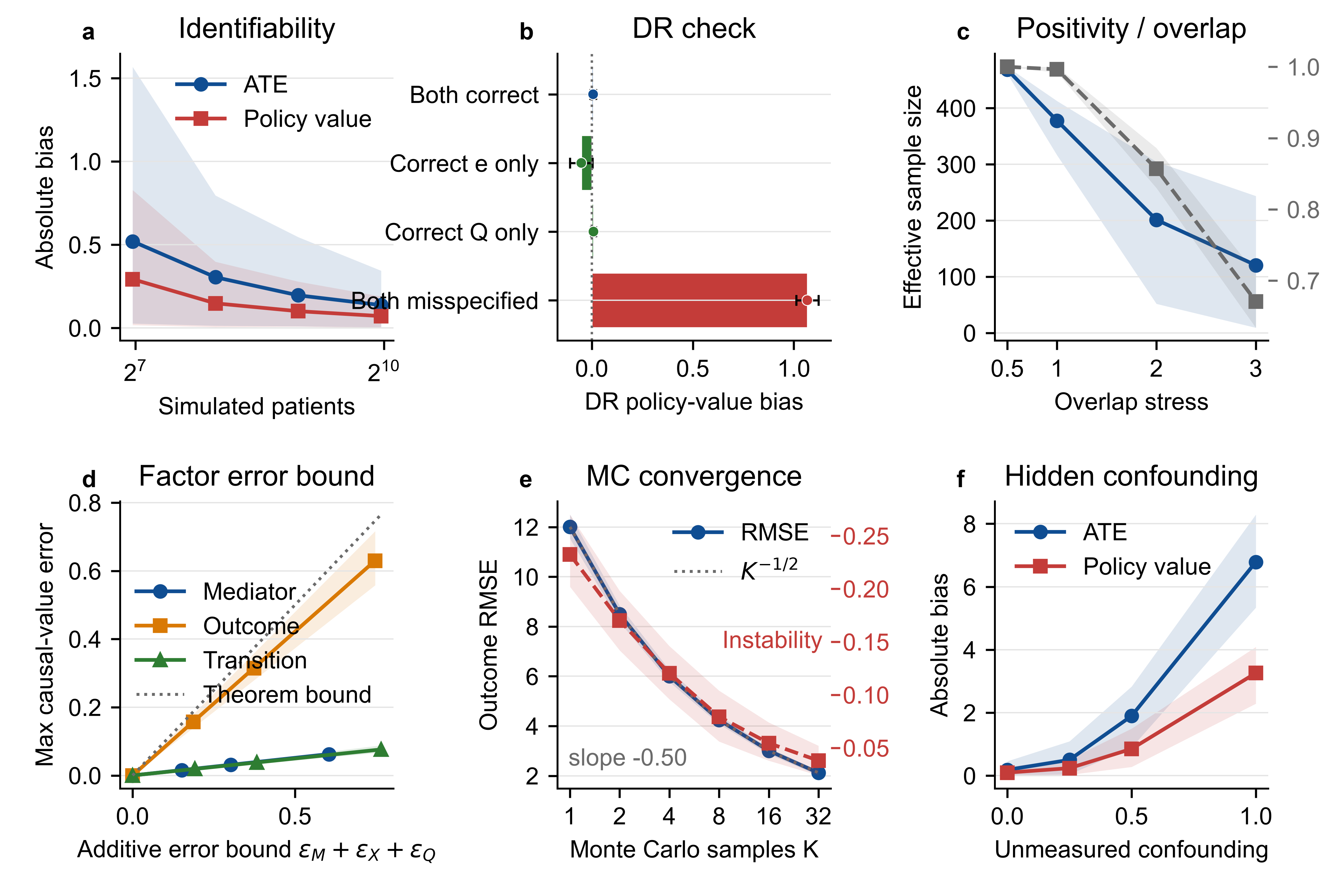}
\caption{\textbf{Known-ground-truth semi-synthetic validation.}
(a) ATE and policy-value recovery with increasing simulated sample size.
(b) Double-robustness identity under correct or misspecified propensity and
outcome nuisances. (c) Positivity stress. (d) Factor-wise causal-value bound.
(e) Monte Carlo convergence with sampled-world count. (f) Sensitivity to an
unmeasured common cause.  The experiment tests the behavior of the estimators
and theoretical bounds under the simulation SCM, not clinical efficacy.}
\label{fig:s5-semisynthetic}
\end{figure}

\subsection{Primary Factual-Only AIPW Policy Value}

The primary policy-value analysis separates policy construction from value
estimation.  CaMeWM's support-restricted policy is frozen from its OOF planning
outputs, whereas the AIPW outcome nuisance is fit fold-locally using only
pre-treatment clinical history, factual regimen, and observed survival.  It
does not consume generated worlds, response descriptors, pseudo masks, or
post-treatment CT.  On all 79 OOF patients, the factual-only AIPW estimate is
113.22 weeks (95\% CI 102.34--124.84), compared with the factual observed-regimen
value of 103.62 weeks.  The estimated gain is 9.60 weeks (95\% CI 0.61--18.13).
The minimum recommended-action support is 0.363, policy ESS is 37.4, and the
largest correction weight is 2.66 under a prespecified cap of 10; no correction
weight is clipped.  These results provide an assumption-conditional,
retrospective policy-value signal, not prospective evidence of treatment
benefit.

\begin{table}[htbp]
\centering
\small
\caption{Primary factual-only AIPW evaluation of the frozen CaMeWM policy.
The outcome nuisance uses pre-treatment clinical history, factual regimen, and
observed survival only.}
\label{tab:factual-aipw}
\begin{tabular}{lr}
\toprule
Quantity & OOF estimate \\
\midrule
Patients & 79 \\
Factual observed-regimen RMST (weeks) & 103.62 \\
Factual-only AIPW policy value (weeks) & 113.22 [102.34, 124.84] \\
Factual-only AIPW gain (weeks) & 9.60 [0.61, 18.13] \\
Policy effective sample size & 37.4 \\
Mean recommended-action support & 0.597 \\
Minimum recommended-action support & 0.363 \\
Maximum correction weight / cap & 2.66 / 10 \\
Correction-weight clipping rate & 0.000 \\
\bottomrule
\end{tabular}
\end{table}

\subsection{Support-Threshold and Weight-Cap Sensitivity}

The primary observational analysis locks the support threshold at
$\eta=0.05$ and the inverse-probability weight cap at 10.  To assess the
dependence of the historical generated-state plug-in diagnostic on these design choices, we sweep
$\eta\in\{0.01,0.025,0.05,0.10,0.20\}$ and caps in
$\{2,5,10,20\}$ using the same five-fold OOF predictions.  Figure~\ref{fig:s6-sensitivity}
reports the resulting plug-in RMST diagnostic and effective sample size.  At the
locked $\eta=0.05$, 77/79 patients retain both observed- and
recommended-regimen support for this diagnostic.  This is deliberately more
stringent than the planner's abstention rule, which only requires at least one
candidate regimen to be supported.  Thus, the sensitivity sweep characterizes
the supported evaluation cohort and weight stability; it does not re-fit the
models, re-optimize the policy, or replace the primary factual-only AIPW
analysis above.

\begin{figure}[htbp]
\centering
\includegraphics[width=0.88\linewidth]{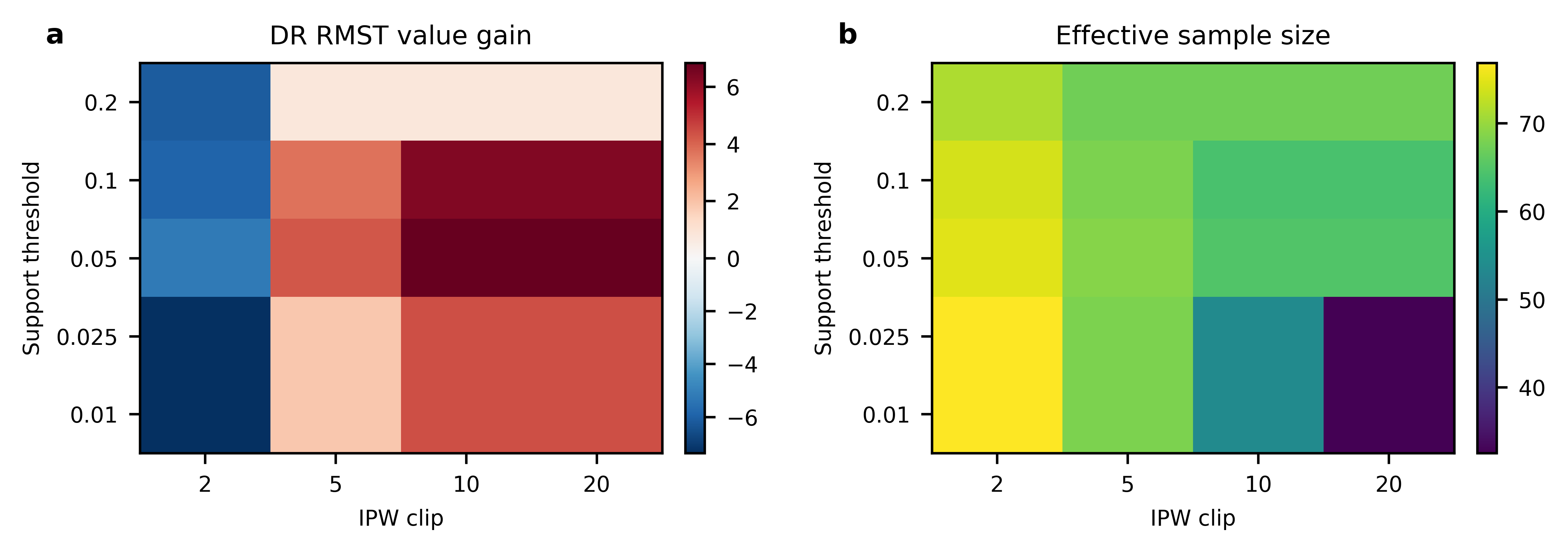}
\caption{\textbf{Support-threshold and IPW-cap sensitivity of a historical
generated-state plug-in diagnostic.}
Each cell uses the fixed five-fold OOF planning and propensity predictions.
Rows vary the joint observed/recommended support threshold $\eta$ and columns
vary the IPW cap.  Panel (a) reports the descriptive plug-in RMST value gain, while
panel (b) reports effective sample size.  The locked primary setting is
$\eta=0.05$ with cap 10.  This analysis assesses overlap and weight sensitivity;
it is not the primary factual-only AIPW estimand and does not convert an
observational interval into evidence of clinical benefit.}
\label{fig:s6-sensitivity}
\end{figure}

\clearpage

\clearpage
\phantomsection
\pdfbookmark[0]{References}{references}
\bibliographystyle{plainnat}
\bibliography{references}

@article{alinezhad2026cdm,
  author={Alinezhad, Farbod and Cao, Jianfei and Young, Gary J. and Post, Brady},
  title={Causal Diffusion Models for Counterfactual Outcome Distributions in Longitudinal Data},
  journal={arXiv preprint arXiv:2604.12992}, year={2026}
}

@article{austin2015iptw,
  author={Austin, Peter C. and Stuart, Elizabeth A.},
  title={Moving towards Best Practice when Using Inverse Probability of Treatment Weighting (IPTW) Using the Propensity Score to Estimate Causal Treatment Effects in Observational Studies},
  journal={Statistics in Medicine}, year={2015}, volume={34}, number={28}, pages={3661--3679}
}

@article{bang2005dr,
  author={Bang, Heejung and Robins, James M.},
  title={Doubly Robust Estimation in Missing Data and Causal Inference Models},
  journal={Biometrics}, year={2005}, volume={61}, number={4}, pages={962--973}
}

@article{chernozhukov2018dml,
  author={Chernozhukov, Victor and Chetverikov, Denis and Demirer, Mert and Duflo, Esther and Hansen, Christian and Newey, Whitney and Robins, James},
  title={Double/Debiased Machine Learning for Treatment and Structural Parameters},
  journal={The Econometrics Journal}, year={2018}, volume={21}, number={1}, pages={C1--C68}
}

@article{cinquini2025causal,
  title={A practical approach to causal inference over time},
  author={Cinquini, Martina and Beretta, Isacco and Ruggieri, Salvatore and Valera, Isabel},
  journal={Proceedings of the AAAI Conference on Artificial Intelligence},
  volume={39},
  number={14},
  pages={14832--14839},
  year={2025}

}

@article{cox1972regression,
  author={Cox, David R.}, title={Regression Models and Life-Tables}, journal={Journal of the Royal Statistical Society: Series B (Methodological)}, year={1972}, volume={34}, number={2}, pages={187--202}
}

@article{ding2025clarity,
  author={Ding, Tianxingjian and Zou, Yuanhao and Chen, Chen and Shah, Mubarak and Tian, Yu},
  title={{CLARITY}: Medical World Model for Guiding Treatment Decisions by Modeling Context-Aware Disease Trajectories in Latent Space},
  journal={arXiv preprint arXiv:2512.08029}, year={2025}
}

@inproceedings{frauen2023deepace,
  author={Frauen, Dennis and Hatt, Tobias and Melnychuk, Valentyn and Feuerriegel, Stefan},
  title={Estimating Average Causal Effects from Patient Trajectories},
  booktitle={Proceedings of the AAAI Conference on Artificial Intelligence}, year={2023}, volume={37}, pages={7586--7594}
}

@article{ha2018worldmodels,
  author={Ha, David and Schmidhuber, J{\"u}rgen},
  title={World Models},
  journal={arXiv preprint arXiv:1803.10122},
  year={2018}
}

@article{hafner2023mastering,
  author = {Hafner, Danijar and Pasukonis, Jurgis and Ba, Jimmy and Lillicrap, Timothy},
  title = {Mastering Diverse Domains through World Models},
  journal = {arXiv preprint arXiv:2301.04104},
  year = {2023},
  doi = {10.48550/arXiv.2301.04104},
  url = {https://arxiv.org/abs/2301.04104}
}

@book{hernan2020whatif,
  author={Hern{\'a}n, Miguel A. and Robins, James M.}, title={Causal Inference: What If}, publisher={Chapman and Hall/CRC}, year={2020}
}

@article{horii2024uncertainty,
  title={Uncertainty quantification in heterogeneous treatment effect estimation with gaussian-process-based partially linear model},
  author={Horii, Shunsuke and Chikahara, Yoichi},
  journal={Proceedings of the AAAI Conference on Artificial Intelligence},
  volume={38},
  number={18},
  pages={20420--20429},
  year={2024}
}

@article{koju2025surgical,
  author = {Koju, Saurabh and Bastola, Saurav and Shrestha, Prashant and Amgain, Sanskar and Shrestha, Yash Raj and Poudel, Rudra P. K. and Bhattarai, Binod},
  title = {Surgical Vision World Model},
  journal = {arXiv preprint arXiv:2503.02904},
  year = {2025},
  doi = {10.48550/arXiv.2503.02904},
  url = {https://arxiv.org/abs/2503.02904}
}

@article{kraljevic2024foresight,
  author = {Kraljevic, Zeljko and Bean, Dan and Shek, Anthony and Bendayan, Rebecca and Hemingway, Harry and Au Yeung, Joshua and Deng, Alexander and Baston, Alfred and Ross, Jack and Idowu, Esther and Teo, James T. and Dobson, Richard J. B.},
  title = {{Foresight}---a Generative Pretrained Transformer for Modelling of Patient Timelines Using Electronic Health Records: A Retrospective Modelling Study},
  journal = {The Lancet Digital Health},
  year = {2024},
  volume = {6},
  number = {4},
  pages = {e281--e290},
  doi = {10.1016/S2589-7500(24)00025-6},
  url = {https://doi.org/10.1016/S2589-7500(24)00025-6}
}

@inproceedings{lee2018deephit,
  author={Lee, Changhee and Zame, William and Yoon, Jinsung and van der Schaar, Mihaela}, title={DeepHit: A Deep Learning Approach to Survival Analysis with Competing Risks}, booktitle={Proceedings of the AAAI Conference on Artificial Intelligence}, year={2018}, volume={32}
}

@article{li2024cfi,
 title={Causal representation learning via counterfactual intervention},
  author={Li, Xiutian and Sun, Siqi and Feng, Rui},
  journal={Proceedings of the AAAI conference on artificial intelligence},
  volume={38},
  number={4},
  pages={3234--3242},
  year={2024}
}

@article{moawad2023hcctace,
  author={Moawad, Ahmed W. and Morshid, Ahmed and Khalaf, Ahmed M. and Elmohr, Mohamed M. and Hazle, John D. and Fuentes, David and Badawy, Mohamed and Kaseb, Ahmed and Hassan, Mohamed and Mahvash, Armeen and Szklaruk, Janio and Qayyum, Ali and Abusaif, Riham and Bennett, William C. and Nolan, Timothy and Camp, Jason J. and Elsayes, Khaled M.},
  title={Multimodality Annotated Hepatocellular Carcinoma Data Set Including Pre- and Post-TACE with Imaging Segmentation}, journal={Scientific Data}, year={2023}, volume={10}, pages={33}
}

@book{pearl2009causality,
  title={Causality},
  author={Pearl, Judea},
  year={2009},
  publisher={Cambridge university press}
}

@article{royston2013rmst,
  author={Royston, Patrick and Parmar, Mahesh K. B.}, title={Restricted Mean Survival Time: An Alternative to the Hazard Ratio for the Design and Analysis of Randomized Trials with a Time-to-Event Outcome}, journal={BMC Medical Research Methodology}, year={2013}, volume={13}, pages={152}
}

@inproceedings{shalit2017ite,
  author={Shalit, Uri and Johansson, Fredrik D. and Sontag, David}, title={Estimating Individual Treatment Effect: Generalization Bounds and Algorithms}, booktitle={Proceedings of the 34th International Conference on Machine Learning}, year={2017}, volume={70}, pages={3076--3085}
}

@article{uno2011cindex,
  author={Uno, Hajime and Cai, Tianxi and Pencina, Michael J. and D'Agostino, Ralph B. and Wei, L. J.}, title={On the C-Statistics for Evaluating Overall Adequacy of Risk Prediction Procedures with Censored Survival Data}, journal={Statistics in Medicine}, year={2011}, volume={30}, number={10}, pages={1105--1117}
}

@article{wang2004ssim,
  author={Wang, Zhou and Bovik, Alan C. and Sheikh, Hamid R. and Simoncelli, Eero P.}, title={Image Quality Assessment: From Error Visibility to Structural Similarity}, journal={IEEE Transactions on Image Processing}, year={2004}, volume={13}, number={4}, pages={600--612}
}

@inproceedings{yang2025mewm,
  author={Yang, Y. and Wang, Z.-Y. and Liu, Q. and Sun, S. and Wang, K. and Chellappa, R. and Zhou, Z. and Yuille, A. and Zhu, L. and Zhang, Y.-D. and Chen, J.}, title={Medical World Model}, booktitle={Proceedings of the IEEE/CVF International Conference on Computer Vision}, year={2025}, pages={8319--8329}
}

@inproceedings{yue2025chexworld,
  author = {Yue, Yang and Wang, Yulin and Tao, Chenxin and Liu, Pan and Song, Shiji and Huang, Gao},
  title = {{CheXWorld}: Exploring Image World Modeling for Radiograph Representation Learning},
  booktitle = {Proceedings of the IEEE/CVF Conference on Computer Vision and Pattern Recognition},
  year = {2025},
  pages = {20778--20788},
  url = {https://openaccess.thecvf.com/content/CVPR2025/html/Yue_CheXWorld_Exploring_Image_World_Modeling_for_Radiograph_Representation_Learning_CVPR_2025_paper.html}
}

@inproceedings{yue2025echoworld,
  author = {Yue, Yang and Wang, Yulin and Jiang, Haojun and Liu, Pan and Song, Shiji and Huang, Gao},
  title = {{EchoWorld}: Learning Motion-Aware World Models for Echocardiography Probe Guidance},
  booktitle = {Proceedings of the IEEE/CVF Conference on Computer Vision and Pattern Recognition},
  year = {2025},
  pages = {25993--26003},
  url = {https://openaccess.thecvf.com/content/CVPR2025/html/Yue_EchoWorld_Learning_Motion-Aware_World_Models_for_Echocardiography_Probe_Guidance_CVPR_2025_paper.html}
}
\end{document}